\documentclass[11pt]{article}

\usepackage[ulibertine=true]{uchitra_1}

\usepackage[utf8]{inputenc} 
\usepackage[T1]{fontenc}    
\usepackage{hyperref}       
\usepackage{url}            
\usepackage{booktabs}       
\usepackage{amsfonts}       
\usepackage{nicefrac}       
\usepackage{microtype}      

\usepackage{wrapfig}
\usepackage{bbm} 
\usepackage{enumitem} 

\usepackage{amsmath, amssymb, amsthm, mathtools, bm, caption, subcaption, multirow, url, hyperref}
\usepackage[usenames,dvipsnames]{xcolor}

\DeclareMathOperator*{\argmax}{argmax}

\newcommand*{\iid}{\overset{\text{i.i.d.}}{\sim}}

\newcommand{\erdosrenyi}{Erd\H{o}s-R\'enyi{}}

\makeatletter
\def\thmhead@plain#1#2#3{%
  \thmname{#1}\thmnumber{\@ifnotempty{#1}{ }\@upn{#2}}%
  \thmnote{{\the\thm@notefont#3}}}
\let\thmhead\thmhead@plain
\makeatother

\DeclareMathOperator{\probdist}{ASD}
\DeclareMathOperator{\gmm}{GMM}

\newcommand*{\prob}{ASD{}}

\newcommand{\tGMM}{\text{GMM}}

\newcommand{\mle}[1]{\widehat{A}_{#1}}
\newcommand{\mleS}{\widehat{A}_{\mathcal{S}}}
\newcommand{\containA}{\breve{\mathcal{S}}(A)}
\newcommand{\mleT}{\widehat{T}}

\newcommand{\asdS}{\probdist_{\mathcal{S}}(A,\mu)}
\newcommand{\alphaGMM}{\widehat{\alpha}_{\tGMM}}
\newcommand{\muGMM}{\widehat{\mu}_{\tGMM}}

\newcommand{\likX}{L_{\alpha+\kappa, \mu+\tau}(\mathbf{X})}

\makeatletter
\newcommand\email[2][]%
   {\newaffiltrue\let\AB@blk@and\AB@pand
      \if\relax#1\relax\def\AB@note{\AB@thenote}\else\def\AB@note{\relax}%
        \setcounter{Maxaffil}{0}\fi
      \begingroup
        \let\protect\@unexpandable@protect
        \def\thanks{\protect\thanks}\def\footnote{\protect\footnote}%
        \@temptokena=\expandafter{\AB@authors}%
        {\def\\{\protect\\\protect\Affilfont}\xdef\AB@temp{#2}}%
         \xdef\AB@authors{\the\@temptokena\AB@las\AB@au@str
         \protect\\[\affilsep]\protect\Affilfont\AB@temp}%
         \gdef\AB@las{}\gdef\AB@au@str{}%
        {\def\\{, \ignorespaces}\xdef\AB@temp{#2}}%
        \@temptokena=\expandafter{\AB@affillist}%
        \xdef\AB@affillist{\the\@temptokena \AB@affilsep
          \AB@affilnote{}\protect\Affilfont\AB@temp}%
      \endgroup
       \let\AB@affilsep\AB@affilsepx
}
\makeatother

\usepackage{authblk}
\title{\vspace{-4em}Quantifying and Reducing Bias in Maximum Likelihood Estimation of Structured Anomalies}
\author[1]{Uthsav Chitra}
\author[1]{Kimberly Ding}
\author[2]{Jasper C.H. Lee}
\author[1]{Benjamin J. Raphael}
\affil[1]{Department of Computer Science, Princeton University}
\email{\texttt{\{uchitra,kding,braphael\}@princeton.edu}\vspace{0.75em}}
\affil[2]{Department of Computer Science, Brown University}
\email{\texttt{jasperchlee@brown.edu}}
\date{\vspace{-3em}}                     
\begin{document}


\maketitle

\begin{abstract}
Anomaly estimation, or the problem of finding a subset of a dataset that differs from the rest of the dataset, is a classic problem in machine learning and data mining.
In both theoretical work and in applications, the anomaly is assumed to have a specific structure defined by membership in an \emph{anomaly family}.
For example, in temporal data the anomaly family may be time intervals, while in network data the anomaly family may be connected subgraphs.
The most prominent approach for anomaly estimation is to compute the Maximum Likelihood Estimator (MLE) of the anomaly; however, it was recently observed that for normally distributed data, the MLE is a \emph{biased} estimator for some anomaly families.
In this work, we demonstrate that in the normal means setting, the bias of the MLE depends on the size of the anomaly family. We prove that if the number of sets in the anomaly family that contain the anomaly is sub-exponential, then the MLE is asymptotically unbiased. We also provide empirical evidence that the converse is true: if the number of such sets is exponential, then the MLE is asymptotically biased. 
Our analysis unifies a number of earlier results on the bias of the MLE for specific anomaly families.
Next, we derive a new anomaly estimator using a mixture model, and we prove that our anomaly estimator is asymptotically unbiased regardless of the size of the anomaly family.
We illustrate the advantages of our estimator versus the MLE on disease outbreak and highway traffic data.
\end{abstract}

\section{Introduction}

Anomaly identification --- the discovery of rare, irregular, or otherwise anomalous behavior in data --- is a fundamental problem in machine learning and data mining with numerous applications \cite{Chandola2009}. In temporal/sequential data, applications of anomaly identification include change-point detection and inference \cite{Page1955,Hinkley1970,Adams2007, Zhai2016}; in matrix data, applications include bi-clustering \cite{Hartigan1972,Tanay2005,Kolar2011} and gene expression analysis \cite{Ideker2002,Dittrich2008}; in spatial data, applications include disease outbreak and event detection \cite{Neill2004,Neill2005, Neill2012};
and in network data, applications include large-scale network surveillance \cite{Arias-Castro2011,Sharpnack2013,Sharpnack2013b} and outbreak detection \cite{Wong2003,Leskovec2007}.
In many applications, the anomalous behavior is assumed to have a specific structure described by membership in an \emph{anomaly family}.
For example, in temporal data the anomaly family may be time intervals;
in matrix data the anomaly family may be submatrices; and
in network data the anomaly family may be connected subgraphs.

Anomaly identification can be divided into two different but closely related problems: \emph{anomaly detection} and \emph{anomaly estimation}.  Given a dataset, 
the goal of anomaly detection is to decide whether or not there exists an \emph{anomaly}, or a subset of the data,
that is distributed according to a different probability distribution compared to the rest of the data. The goal of anomaly estimation is to determine the data points in the anomaly.
The distinction between anomaly detection and anomaly estimation is analogous to the distinction between property testing and proper learning in statistical learning theory \cite{Goldreich1998}: just as property testing is ``easier" than proper learning (with difficulty measured by sample complexity), anomaly detection is easier than anomaly estimation (with difficulty measured by the separation between the distributions of the anomaly and the rest of the data).
Different choices of the anomaly family give rise to different versions of the anomaly detection and estimation problems; e.g.
change-point detection versus change-point inference in temporal data \cite{Arias-Castro2005,Hinkley1971,Jeng2010}, or submatrix detection versus submatrix estimation in matrix data \cite{Hajek2017,Butucea2013,Ma2015,Brennan2019,Chen2016,Banks2018,Liu2019,Gamarnik2019}.


Most of the theoretical literature on anomaly detection and estimation focuses on \emph{structured normal means} problems \cite{Sharpnack2013b,Krishnamurthy2016}. In this setting, each data point is drawn from one of two normal distributions, with the data points from the anomaly drawn from the normal distribution with the higher mean; the structure of the anomaly is determined by the anomaly family.  
Normal means problems have a long history in statistics and machine learning as many statistical tests commonly used in  scientific disciplines are asymptotically normal, e.g.  see \cite{Arias-Castro2011,Donoho2004,Cai2007,Kolar2011,Sharpnack2013b,Chen2016,Liu2019}.
In this paper we also focus on the structured normal means setting, but we emphasize that our results algorithms can be readily extended to other probability distributions from the exponential family as in earlier works \cite{Butucea2013,Liu2019}.


The most widely used techniques for both anomaly detection and anomaly estimation problems are likelihood models: the generalized likelihood ratio (GLR) test for the detection problem, and the maximum likelihood estimator (MLE) for the estimation problem. Both the GLR test statistic and the MLE can be expressed using a \emph{scan statistic}, or the maximization of a function across all members of the anomaly family \cite{Kulldorff1997,Glaz2010}.  In fact,  as we note in Proposition \ref{thm:glr_mle_asd}, both the GLR test statistic and the MLE involve the maximization of the \emph{same} function. 

Despite this close relationship between the GLR test and the MLE, the two quantities have different theoretical guarantees for their respective problems.
The GLR test is known to be asymptotically ``near-optimal" for solving the anomaly detection problem across many different anomaly families, including intervals \cite{Arias-Castro2005}, submatrices \cite{Butucea2013}, subgraphs with small cut-size \cite{Sharpnack2013b}, and connected subgraphs \cite{Qian2014}. 
In contrast, the MLE is known to be asymptotically near-optimal for solving the anomaly estimation problem only when the anomaly family is intervals \cite{Jeng2010} or submatrices \cite{Liu2019}.
In fact, \cite{NetMix} recently observed that the MLE is a \emph{biased} estimator of the size of the anomaly when the anomaly family is connected subgraphs of a biological network.
 
These varying results for anomaly estimation across different anomaly families suggest that the bias of the MLE depends on the anomaly family, and thus raise the following two questions: (1) For which anomaly families is the MLE biased? (2) Are there anomaly estimators that are less biased than the MLE?

In this work we address both of these questions.
First, we show that the bias in the MLE depends on the size of the anomaly family.\footnote{All proofs are given in the Appendix.} We prove that if the number of sets in the anomaly family that contain the anomaly is sub-exponential, then the MLE is an asymptotically unbiased estimator. We also provide empirical evidence that the converse is true by examining many common anomaly families including intervals, submatrices, connected subgraphs, and subgraphs with low-cut size.  Our results unify a number of previous results in the literature including the asymptotic optimality of the MLE when the anomaly family is intervals \cite{Jeng2010} or submatrices \cite{Liu2019}, and the observation that the MLE is biased when the anomaly family is connected subgraphs \cite{NetMix}.

Next, we derive a reduced-bias estimator of the anomaly based on a Gaussian mixture model (GMM).  Our estimator is motivated by previous work that models \emph{unstructured} anomalies using GMMs \cite{Cai2007,Donoho2004}.
We prove that our GMM-based estimator is asymptotically unbiased,
regardless of the size of the anomaly family or the number of sets containing the anomaly. We empirically demonstrate the small bias of our estimator for several anomaly families including intervals, submatrices, and connected subgraphs. We illustrate the advantages of our estimator versus the MLE on both disease outbreak data and a highway traffic dataset.




\section{Background: Structured Anomalies and Maximum Likelihood Estimation}
\subsection{Problem Formulation}

Suppose one is given observations $(X_1, \dots, X_n)$, where a subset $A \subseteq \{1, \dots, n\}$ of these observations, the \emph{anomaly}, are drawn from a normal distribution $N(\mu,1)$ with elevated mean and the remaining observations are drawn from the standard normal distribution $N(0,1)$.  Using the notation $[n] = \{1, \dots, n\}$ and $\mathcal{P}_n$ for the power set of $[n]$, or the set of all subsets of $[n]$ for a positive integer $n$, we define the distribution of the observations $(X_1, \dots, X_n)$ as follows.

\begin{namedproblem}[\textbf{Anomalous Subset Distribution (ASD)}]
Let $\mu > 0$, let $\mathcal{S} \subseteq \mathcal{P}_n$ be a family of subsets of $[n]$ and let $A \in \mathcal{S}$.
We say $\mathbf{X} = (X_1, \dots, X_n)$ is distributed according to the \emph{Anomalous Subset Distribution} $\probdist_{\mathcal{S}}(A, \mu)$ provided the $X_i$ are independently distributed as
\begin{equation}
    X_i \sim 
    \begin{cases}
    N(\mu, 1), & \text{if } i \in A,\\
    N(0, 1), & \text{otherwise}.
    \end{cases}
\end{equation}
\end{namedproblem}
The distribution $\probdist_{\mathcal{S}}(A, \mu)$ has three parameters: the \emph{anomaly family} $\mathcal{S}$, the \emph{anomaly} $A$, and the \emph{mean} $\mu$.  

\begin{figure*}[t]
    \centering
\includegraphics[width=\linewidth]{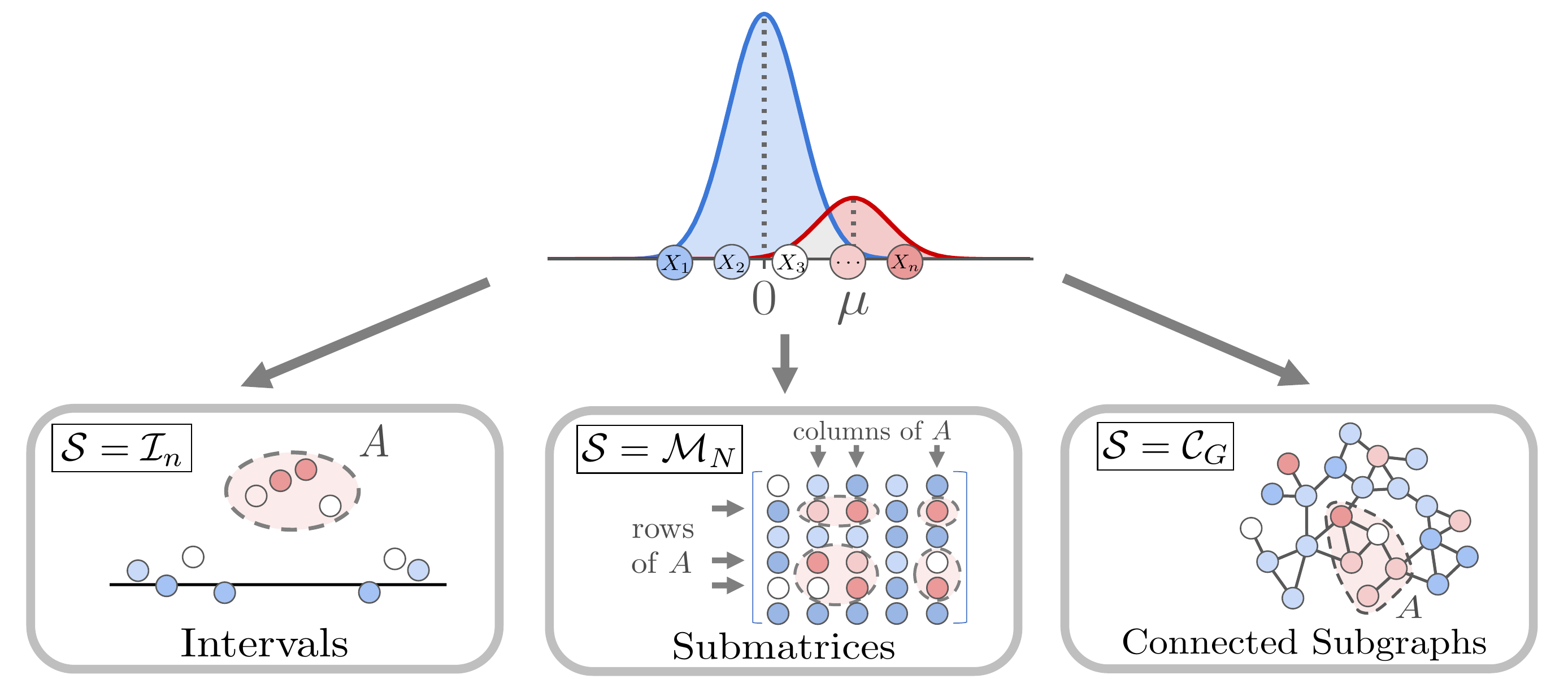}
    \caption{Observations $(X_1, \dots, X_n)$ from the Anomalous Subset Distribution $\probdist_{\mathcal{S}}(A, \mu)$ for three anomaly families $\mathcal{S}$.}
    \label{fig:overview}
\end{figure*}

The goal of anomaly estimation is to learn the anomaly $A$, given data $\mathbf{X} \sim \probdist_{\mathcal{S}}(A, \mu)$ and anomaly family $\mathcal{S}$. 
We formalize this problem as the following estimation problem.

\begin{namedproblem}[\textbf{\prob{} Estimation Problem}]
Given $\mathbf{X} = (X_1, \dots, X_n) \sim \probdist_{\mathcal{S}}(A, \mu)$ and $\mathcal{S}$, find $A$.
\end{namedproblem}

A related problem is the decision problem of deciding whether or not data $\mathbf{X}$ contains an anomaly.  We formalize this problem as the following hypothesis testing problem.

\begin{namedproblem}[\textbf{\prob{} Detection Problem}]
Given $\mathbf{X} = (X_1, \dots, X_n) \sim \probdist_{\mathcal{S}}(A, \mu)$ and $\mathcal{S}$, test between the hypotheses $H_0: A = \varnothing$ and $H_1: A \neq \varnothing$.
\end{namedproblem}

The ASD Detection and Estimation Problems are also called \emph{structured normal means} problems \cite{Krishnamurthy2016,Sharpnack2013b},  where the structure comes from the choice of the anomaly family $\mathcal{S}$.

Many well-known problems in machine learning correspond to the ASD Detection and Estimation Problems for different anomaly families $\mathcal{S}$. 
In particular, we note the following examples.
\begin{itemize}
    \item $\mathcal{S} = \mathcal{I}_n$, the set of all intervals $\{i, i+1, \dots, j\} \subseteq [n]$. We call $\mathcal{I}_n$ the \emph{interval} family, and we call $\probdist_{\mathcal{I}_n}(A, \mu)$ the \emph{interval ASD}.
    The interval ASD is used to model change-points, or abrupt changes, in sequential data including time-series and DNA sequences \cite{Hinkley1970,Hinkley1971,Basseville1993,Jeng2010}.
    
    \item $\mathcal{S} = \mathcal{C}_G$, the set of all connected subgraphs of a graph $G = (V, E)$ with vertices $V = \{1, \dots, n\}$. 
    We call $\mathcal{C}_G$ the \emph{connected} family, and we call $\probdist_{\mathcal{C}_G}(A, \mu)$ the \emph{connected ASD}. 
    The connected ASD is used to model anomalous behavior in different types of networks including social networks, sensor networks and biological networks \cite{Qian2014,Aksoylar2017,Ideker2002,NetMix}. Note that the interval family  $\mathcal{I}_n$ is a special case of the connected family $\mathcal{C}_{P_n}$ for the path graph $P_n$ with $n$ vertices.
    
    \item $\mathcal{S} = \mathcal{T}_{G, \rho}$, the set of all subgraphs $H$ of a graph $G$ with $|\{(i, j) \in E : i \in H, j \not\in H\}| \leq \rho$. We call $\mathcal{T}_{G, \rho}$ the \emph{graph cut family}, and we call $\probdist_{\mathcal{T}_{G, \rho}}(A, \mu)$ the \emph{graph cut ASD}.  The graph cut ASD is also used to model anomalous behavior in networks \cite{Sharpnack2013,Sharpnack2013b,Sharpnack2016}.
    \item $\mathcal{S} = \mathcal{E}_{G,\delta}$, the set of all subgraphs $H$ of a graph $G$ with edge-density at least $\delta$.  We call $\mathcal{E}_{G, \delta}$ the \emph{edge-dense family}, and we call $\probdist_{\mathcal{E}_{G, \delta}}(A, \mu)$ the \emph{edge-dense ASD}.  The edge-dense ASD is also used to model anomalous behavior in networks \cite{Cadena2018b}.
    
        \item $\mathcal{S} = \mathcal{M}_N$, the set of all submatrices of a square matrix $N$ with $n$ entries (each observation $X_i$ corresponds to an entry of $N$). We call $\mathcal{M}_N$ the \emph{submatrix family}, and we call $\probdist_{\mathcal{M}_N}(A, \mu)$  the \emph{submatrix ASD}. 
The clustering literature often uses the submatrix ASD to model biclusters in matrix data
\cite{Kolar2011,Butucea2013,Brennan2019,Liu2019}.

      \item $\mathcal{S} = \mathcal{B}_{P, \epsilon}$, the set of all $\epsilon$-balls of points $P=\{p_1, \dots, p_n\} \subseteq \mathbb{R}^d$ in space. 
        We call  $\mathcal{B}_{P, \epsilon}$ the \emph{$\epsilon$-ball} family.
The spatial scan statistic is a standard tool for solving the ASD Estimation Problem with the $\epsilon$-ball family $\mathcal{B}_{P, \epsilon}$ \cite{Kulldorff1997,Glaz2010}.
    
    \item $\mathcal{S} = \mathcal{P}_n$, the power set of $\{1, \dots, n\}$. We call $\mathcal{P}_n$ the \emph{unstructured family}, and we call $\probdist_{\mathcal{P}_n}(A, \mu)$ the \emph{unstructured ASD}. 
\end{itemize}

\subsection{Maximum Likelihood Anomaly Estimation}
\label{subsec:mle}

A standard approach in statistics for solving a hypothesis testing problem is to use the generalized likelihood ratio (GLR) test, which the Neyman-Pearson lemma \cite{Lehmann2005} shows is the most powerful test for any significance level.  Likewise, a standard approach for solving an estimation problem is to compute a maximum likelihood estimator (MLE). For the ASD Detection and Estimation problems, the GLR test statistic and the MLE, respectively, have explicit formulas that involve the maximization of the same function, $\Gamma(S) = \frac{1}{\sqrt{S}} \sum_{v\in S} X_v$. We write out these formulas below; see \cite{Arias-Castro2011,Sharpnack2013b,NetMix} for proofs.

\begin{prop}
\label{thm:glr_mle_asd}
Let $\mathbf{X} \sim \probdist_{\mathcal{S}}(A, \mu)$ be distributed according to the ASD. The Generalized Likelihood Ratio (GLR) test statistic $\widehat{t}_{\mathcal{S}}$ for the ASD Detection Problem is
\begin{equation}
\label{eq:glr_test_stat}
\widehat{t}_{\mathcal{S}} = \max_{S \in \mathcal{S}} \; \Gamma(S) = \max_{S \in \mathcal{S}} \frac{1}{\sqrt{|S|}} \sum_{v\in S} X_v.
\end{equation}
The Maximum Likelihood Estimator (MLE) $\mleS$ of the anomaly $A$ is
\begin{equation}
\label{eq:mle_def}
\mleS = \argmax_{S \in \mathcal{S}} \; \Gamma(S) = \argmax_{S \in \mathcal{S}} \frac{1}{\sqrt{|S|}} \sum_{v\in S} X_v.
\end{equation}
\end{prop}


A key question in the statistics literature is: for what anomaly families $\mathcal{S}$ and means $\mu$ (i.e. parameters of the ASD) do the GLR test and the MLE solve the ASD Detection and Estimation problems, respectively?

For many anomaly families $\mathcal{S}$, 
it has been shown that the GLR test is asymptotically ``near-optimal". This means that
there exists a value $\mu_{\text{detect}} > 0$ such that the following is true: if $\mu \geq \mu_{\text{detect}}$ then
the GLR test asymptotically solves the ASD Detection Problem with the probability of a type 1 or type 2 error going to $0$ as $n \to \infty$, while if $\mu$ is not much smaller than $\mu_{\text{detect}}$ then there does not exist any test with such guarantees on its type 1 or type 2 error probabilities.
Anomaly families $\mathcal{S}$ for which the GLR test is known to be asymptotically near-optimal include the interval family $\mathcal{S} = \mathcal{I}_n$ \cite{Arias-Castro2005}, the submatrix family $\mathcal{S} = \mathcal{M}_N$ \cite{Butucea2013}, the graph cut family $\mathcal{S} = \mathcal{T}_{G, \rho}$ \cite{Sharpnack2013,Sharpnack2013b}, and the connected family $\mathcal{S} = \mathcal{C}_G$ \cite{Qian2014,Qian2014b}.

For a few anomaly families $\mathcal{S}$, the MLE $\mleS$ has also been shown to optimally solve the ASD Estimation Problem. For the interval family $\mathcal{S} = \mathcal{I}_n$,
\cite{Jeng2010} showed that if $\mu \geq \mu_{\text{detect}}$, then 
$\displaystyle\lim_{n\to\infty} P(\mle{\mathcal{I}_n} \neq A) = 0$.  \cite{Liu2019} proved an analogous result for the submatrix family $\mathcal{S} = \mathcal{M}_N$ using a regularized MLE.
Note that these results require $\mu \geq \mu_{\text{detect}}$, as it is not possible to estimate the anomaly without first detecting the anomaly's presence.




The MLE $\mleS$ is also used in the bioinformatics literature to solve the ASD Estimation Problem for the connected family $\mathcal{S} = \mathcal{C}_G$, where $G$ is a biological network \cite{Ideker2002, Dittrich2008}.
However, the MLE $\mle{\mathcal{C}_G}$ for the connected family $\mathcal{C}_G$ does not have any theoretical guarantees, 
unlike the previously mentioned results for the interval and submatrix families.
In fact, \cite{Nikolayeva2018} and \cite{NetMix} empirically observed that the size $|\mle{\mathcal{C}_G}|$ of the MLE $\mle{\mathcal{C}_G}$ is a biased estimate of the size $|A|$ of the anomaly, in the sense of the following definition.

\begin{defn}
Given data $\mathbf{X}=(X_1, \dots, X_n)$,
let $\hat \theta = \hat \theta(\mathbf{X})$ be an estimator of a parameter $\theta$ of the distribution of $\mathbf{X}$.
The quantity $\text{Bias}_{\theta}(\hat \theta) = E[\hat \theta] - \theta$ is the \emph{bias} of the estimator $\hat \theta$.
We say that $\hat \theta$ is a \emph{biased} estimator of $\theta$ if $\text{Bias}_{\theta}(\hat \theta) \neq 0$, and that $\hat \theta$ is an \emph{unbiased} estimator of $\theta$ otherwise.
When it is clear from context, we omit the subscript $\theta$ and write $\text{Bias}(\hat \theta)$ for the bias of estimator $\hat\theta$.
\end{defn}

\cite{NetMix} also empirically observed a similar bias for the MLE $\mle{\mathcal{P}_n}$  for the unstructured family $\mathcal{S} = \mathcal{P}_n$.

We note that while many papers in statistics study the bias of MLEs for different distributions (e.g. \cite{Firth1993,Mardia1999,Giles2013}), to our knowledge, the bias of the MLE $\mleS$ for the ASD has previously only been studied in \cite{NetMix}.

\section{Relating MLE Bias to Size of the Anomaly Family}
\label{sec:mle_bias}


The observations in the previous section lead to the following question: 
for which anomaly families $\mathcal{S}$ is the size $|\mleS|$ of the MLE $\mleS$ a biased estimate of the size $|A|$ of the anomaly $A$?
In this section, we provide theoretical and experimental evidence that the key quantity that determines the bias of the MLE $\mleS$ is the quantity $\containA = \{S \in \mathcal{S} : S \supseteq A\}$, or the collection of sets in the anomaly family $\mathcal{S}$
that contain the anomaly $A$.

%
%

First, we show that if the number $|\containA|$ of sets containing the anomaly $A$ is sub-exponential in $n$, then the size $|\mleS|$ of the MLE $\mleS$ is asymptotically unbiased.

\begin{theorem}
\label{thm:subexp}
Let $\mathbf{X} = (X_1, \dots, X_n) \sim \probdist_{\mathcal{S}}(A, \mu)$ where $
|\mathcal{S}| = \Omega(n)$ and $|A| = \alpha n$ with $0 < \alpha < 0.5$.
Suppose $|\containA|$ is sub-exponential in $n$. If $\displaystyle\lim_{n\to\infty} P(A \subseteq \mleS) = 1$,
then $\displaystyle\lim_{n\to\infty} \textnormal{Bias}(|\mleS| / n) = 0$.
\end{theorem}

A key component of our proof of Theorem \ref{thm:subexp} is the following Lemma relating the $\text{Bias}(|\mleS| / n)$ of the MLE $\mleS$ to the number $|\containA|$ of sets containing the anomaly $A$.

\begin{namedproblem}[\textbf{Lemma}]
\label{lem:main_lem}
Let $\mathbf{X} \sim \probdist_{\mathcal{S}}(A, \mu)$ where $|\mathcal{S}| = \Omega(n)$ and $|A| = \alpha n$ with $0 < \alpha < 0.5$. Assume $\displaystyle\lim_{n\to\infty} P(A \subseteq \mleS) = 1$. If $n$ is sufficiently large and $\text{Bias}(|\mleS| / n) \geq \gamma$, then 
\begin{equation}
|\containA| \geq (C_{\mu, \gamma, \alpha})^n \cdot e^{-\Theta(\sqrt{n \log n})},
\end{equation}
where
$C_{\mu, \alpha, \gamma} = \exp\left( \frac{1}{2} \mu^2 \alpha^2\left( \sqrt{1+\frac{\gamma}{4\alpha}} - 1 \right)^2   \right).$
\end{namedproblem}

We make two mild assumptions in Theorem \ref{thm:subexp}. First, we assume the proportion $\alpha = \frac{|A|}{n}$ of anomalous observations is a positive constant independent of $n$.
Second, we assume that the anomaly family $\mathcal{S}$ has size $|\mathcal{S}| = \Omega(n)$; this assumption is satisfied by many commonly-used anomaly families including the interval family $\mathcal{S} = \mathcal{I}_n$, the submatrix family $\mathcal{S} = \mathcal{M}_N$, the connected family $\mathcal{S} = \mathcal{C}_G$ for any graph $G$,
and the unstructured family $\mathcal{S}= \mathcal{P}_n$.


We also require that $\lim_{n\to\infty} P(A \subseteq \mleS) = 1$, which is a technical condition needed for the proof of Theorem \ref{thm:subexp}. We conjecture that this condition can be replaced by the condition $\mu \geq \mu_{\text{detect}}$. This conjecture is based on the empirical observation that if $\mu \geq \mu_{\text{detect}}$, then the MLE $\mleS$ contains most of the elements in $A$ (Figure \ref{fig:justify_conj}), suggesting that the condition $\mu \geq \mu_{\text{detect}}$ is only a slightly weaker condition than $\lim_{n\to\infty} P(A \subseteq \mleS) = 1$.


Theorem \ref{thm:subexp} generalizes earlier results showing that the MLEs $\mle{\mathcal{I}_n}$ and $\mle{\mathcal{M}_N}$ for the interval family $\mathcal{I}_n$ and submatrix family $\mathcal{M}_N$ are asymptotically unbiased \cite{Jeng2010,Liu2019} as these families satisfy the conditions of Theorem \ref{thm:subexp}.
Moreover, Theorem \ref{thm:subexp} implies that the regularization of the MLE used in \cite{Liu2019} is not necessary to prove asymptotic unbiasedness (see Appendix \ref{sec:append:submat_regularized}).

Informally, Theorem \ref{thm:subexp} says that if the number $|\containA|$ of subsets that contain the anomaly $A$ is sub-exponential in $n$, then the MLE $\mleS$ is an asymptotically unbiased estimator of the size $|A|$ of the anomaly $A$.
Next, we prove that for the unstructured family $\mathcal{P}_n$, where $|\containA|$ is exponential in $n$, 
the MLE $\mle{\mathcal{P}_n}$is asymptotically biased for all $\mu$. 
This result settles a conjecture posed by \cite{NetMix}.

\begin{figure*}[t]
    \centering
\includegraphics[scale=0.43]{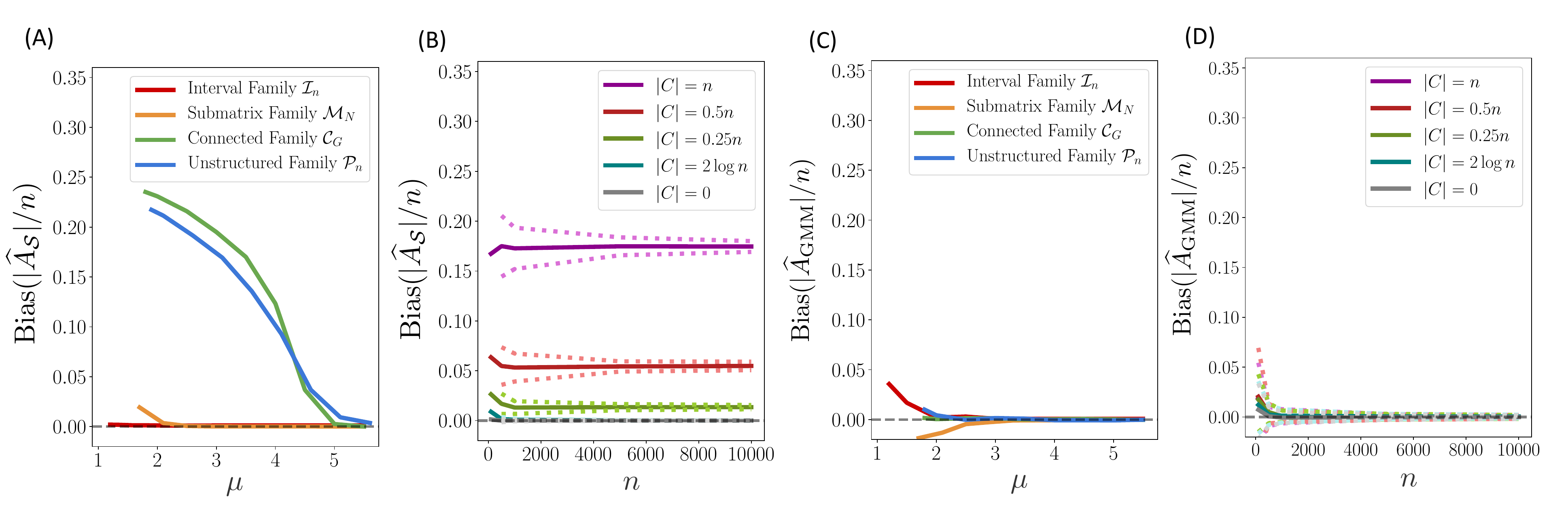}
\caption{Bias in estimators of the size $|A|$ of the anomaly $A$ computed from $n=900$ samples $\mathbf{X} = (X_1, \dots, X_n) \sim \probdist_{\mathcal{S}}(A, \mu)$ from the ASD for different choices of anomaly family $\mathcal{S}$. In each sample, the anomaly $A$ of size $|A| = 0.05n$ is chosen uniformly at random from $\mathcal{S}$.
\textbf{(A)} $\text{Bias}(|\mleS|/n)$ of the MLE vs $\mu$ for means $\mu \geq \mu_{\text{detect}}$.   For small $\mu$, the MLE shows positive bias for the connected family $\mathcal{C}_G$ and unstructured family $\mathcal{P}_n$, consistent with Conjecture \ref{conj:central_claim} and Theorem \ref{thm:unconstr}.
\textbf{(B)} $\text{Bias}(|\mleS|/n)$ vs $n$ for $\mu = 3$ for the connected family $\mathcal{S} = \mathcal{C}_G$, where $G = (V, E)$ is a graph whose vertices $V = P \cup C$ are partitioned into a path graph $P$ and a clique $C$. Dotted lines indicate first and third quartiles in the estimate of the bias. The positive bias for $|C| = \Theta(n)$ does not decrease as $n$ increases, consistent with Conjecture \ref{conj:central_claim}.
\textbf{(C)} $\text{Bias}(|\widehat{A}_{\tGMM}|/n)$ of GMM-based estimator vs $\mu$ for  means $\mu \geq \mu_{\text{detect}}$.  In contrast to (A), the bias is zero for all anomaly families and sufficiently large mean $\mu$, consistent with Corollary \ref{cor:gmm1}. \textbf{(D)} $\text{Bias}(|\widehat{A}_{\tGMM}|/n)$ vs $n$ for $\mu = 3$ and the same connected anomaly family as in (B).  The GMM-based estimator appears to be unbiased for sufficiently large $n$, consistent with Corollary \ref{cor:gmm1}.
}
    \label{fig:mle}
\end{figure*}

 \begin{theorem}
\label{thm:unconstr}
Let $\mathbf{X} \sim \probdist_{\mathcal{P}_n}(A, \mu)$ where $|A| = \alpha n$ with $0 < \alpha < 0.5$.
Then $\lim_{n\to\infty} \textnormal{Bias}(|\mle{\mathcal{P}_n}|/n) > 0$. 
\end{theorem}

We prove Theorem \ref{thm:unconstr} by deriving an explicit formula for the asymptotic bias $\lim_{n\to\infty} \text{Bias}(|\mle{\mathcal{P}_n}| / n)$ of the MLE.


We conjecture that Theorems \ref{thm:subexp} and \ref{thm:unconstr}  describe the only two possible values for the asymptotic bias of the MLE $\mleS$, and furthermore that the technical conditions in Theorem~\ref{thm:subexp} can be relaxed to the simpler condition $\mu \geq \mu_{\text{detect}}$. Thus we conjecture that MLE $\mleS$ is asymptotically biased if and only if $|\containA|$ is exponential in $n$.

\begin{conjecture}
\label{conj:central_claim}
Let $\mathbf{X} \sim \probdist_{\mathcal{S}}(A, \mu)$ with $\mu \geq \mu_{\textnormal{detect}}$.  Then $\lim_{n\to\infty} \textnormal{Bias}(|\mleS| / n) > 0$ if $|\containA|$ is exponential in $n$, and $\lim_{n\to\infty} \textnormal{Bias}(|\mleS| / n) = 0$ otherwise.
\end{conjecture}

Conjecture \ref{conj:central_claim} generalizes Theorems \ref{thm:subexp} and \ref{thm:unconstr}, and is consistent with the prior work noted in Section \ref{subsec:mle} on the bias of the MLE $\mleS$ for different anomaly families $\mathcal{S}$:
\begin{itemize}
\item  $\mathcal{S} = \mathcal{I}_n$, the interval family: $|\containA| \leq |\mathcal{S}| = O(n^2)$ is sub-exponential, so $|\mle{\mathcal{I}_n}|$ is asymptotically unbiased \cite{Jeng2010}.

\item $\mathcal{S} = \mathcal{M}_N$, the submatrix family: $|\containA| \leq |\mathcal{S}| = O(2^{2\sqrt{n}})$ is sub-exponential, so $|\mle{\mathcal{M}_N}|$ is asymptotically unbiased \cite{Liu2019}.

\item $\mathcal{S} = \mathcal{C}_G$, the connected family: When $G$ has minimum degree $3$, $|\containA|$ is exponential \cite{Vince2017}, so $|\mle{\mathcal{C}_G}|$ is asymptotically biased \cite{NetMix}.

\item $\mathcal{S} = \mathcal{P}_n$, the unstructured family: When $|A| < 0.5n$, $|\containA| = 2^{n\left(1-\frac{|A|}{n}\right)} = \Omega\left( 2^{0.5n} \right)$ is exponential, so $|\mle{\mathcal{P}_n}|$ is asymptotically biased.
\end{itemize}



\subsection{Experimental Evidence for Conjecture \ref{conj:central_claim}}
\label{subsec:exp_evidence}

We provide empirical evidence for Conjecture \ref{conj:central_claim} by examining the bias of the MLE for several different anomaly families.  For each anomaly family $\mathcal{S}$, we select an anomaly $A \in \mathcal{S}$ with size $|A| = 0.05n$ uniformly at random from $\mathcal{S}$. We  draw a sample $\mathbf{X} = (X_1, \dots, X_n) \sim \probdist_{\mathcal{S}}(A, \mu)$ with $n=900$ observations, and compute the MLE $\mleS$.
We repeat for $50$ samples to estimate $\text{Bias}(|\mleS|/n)$. We perform this process for a range of means $\mu \geq \mu_{\text{detect}}$ (see Appendix \ref{sec:append:mudetect} for details on empirically calculating $\mu_{\text{detect}}$.)

\begin{wrapfigure}{l}{0.5\textwidth}
    \includegraphics[width=\linewidth]{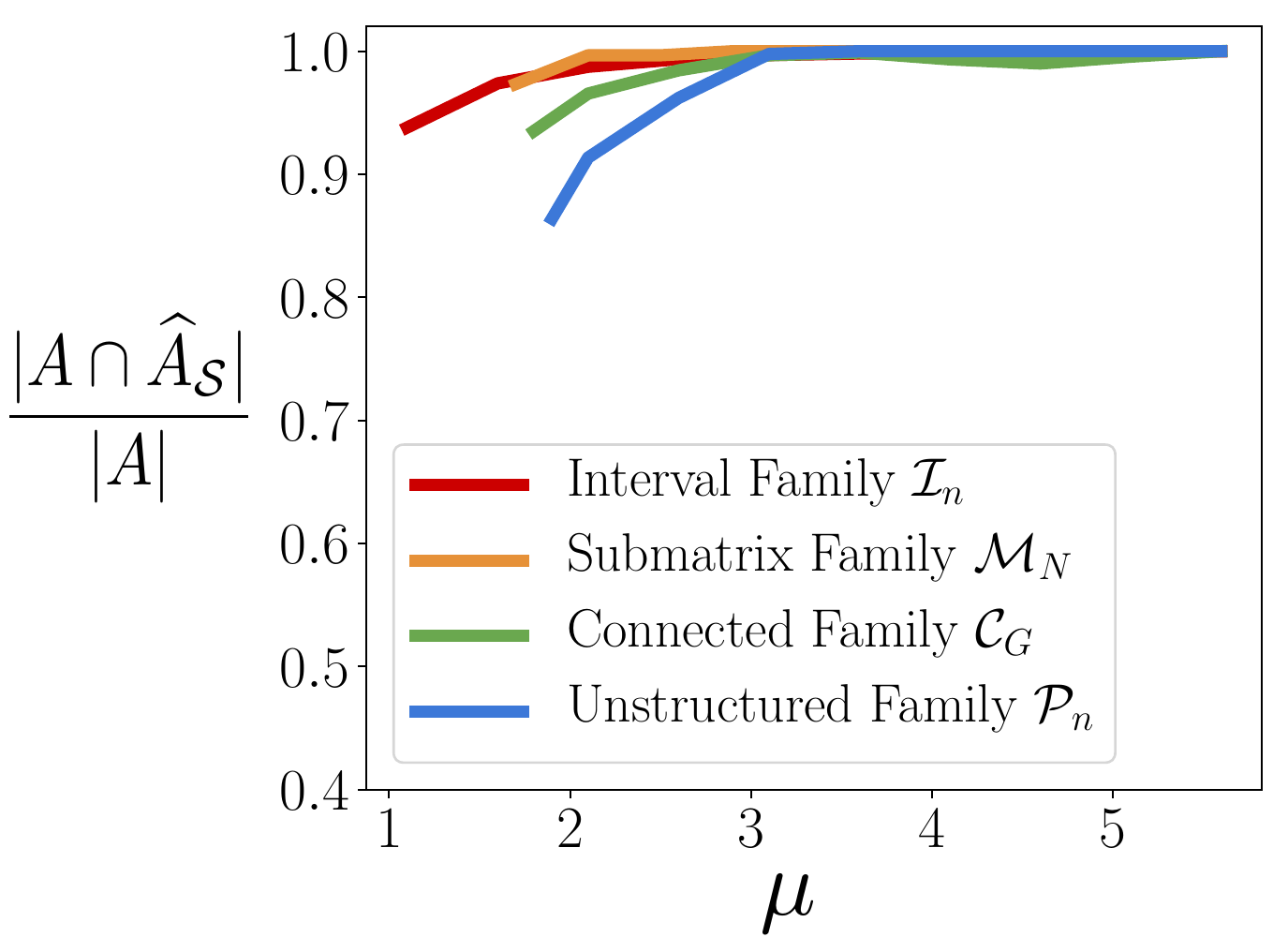}
  \caption{Normalized intersection $\frac{|A \cap \mleS|}{|A|}$ between anomaly $A$ and MLE $\mleS$ for means $\mu \geq \mu_{\text{detect}}$ for different anomaly families $\mathcal{S}$.  Data generated as in Figure \ref{fig:mle}.
For $\mu \geq \mu_{\text{detect}}$ and different anomaly families $\mathcal{S}$, the normalized intersection $\frac{|A \cap \mleS|}{|A|} > 0.85$, i.e. the MLE $\mleS$ contains at least $85\%$ of the elements in the anomaly $A$. This suggests that the condition $\displaystyle\lim_{n\to\infty} P(A\subseteq \mleS)=1$ in Theorem \ref{thm:subexp} is not much stronger than the condition $\mu\geq \mu_{\text{detect}}$.
}
  \label{fig:justify_conj}
\end{wrapfigure}

We compute $\text{Bias}(|\mleS|/n)$ for the following anomaly families: $\mathcal{S} = \mathcal{I}_n$, the interval family; $\mathcal{S} = \mathcal{M}_N$, the submatrix family with matrix $N \in \mathbb{R}^{30\times 30}$; $\mathcal{S} = \mathcal{C}_G$, the connected family with an \erdosrenyi{} random graph $G$ (edge probability $= 0.01$); and $\mathcal{S}=\mathcal{P}_n$, the unstructured family.
$|\containA|$ is sub-exponential for the interval family $\mathcal{I}_n$ and the submatrix family $\mathcal{M}_N$, and is exponential for the connected family $\mathcal{C}_G$ (with high probability \cite{Vince2017}) and for the unstructured family $\mathcal{P}_n$.

For the interval family $\mathcal{I}_n$ and submatrix family $\mathcal{M}_N$, where $|\containA|$ is sub-exponential, we find that $\text{Bias}(|\mleS| /n) \approx 0$ for all means $\mu \geq \mu_{\text{detect}}$ (Figure \ref{fig:mle}A). In contrast, for the connected family $\mathcal{C}_G$ and unstructured family $\mathcal{P}_n$, where $|\containA|$ is exponential, we observe that $\text{Bias}(|\mleS| /n) > 0$ for $\mu \in [\mu_{\text{detect}},5]$ (Figure \ref{fig:mle}A).  (Because $n$ is fixed, the $\text{Bias}(|\mleS| /n)$ will be zero for sufficiently large $\mu$.)
These observations provide evidence in support of Conjecture \ref{conj:central_claim} for these families.  
%
Moreover, although Conjecture \ref{conj:central_claim} is about the $\text{Bias}(|\mleS|/n)$ of the MLE $\mleS$, we also observe that larger $\text{Bias}(|\mleS|/n)$ reduces the F-measure between the anomaly $A$ and the MLE $\mleS$ (see Appendix \ref{sec:append:exps}).

Next, we examine the $\text{Bias}(|\mleS| /n)$ of the MLE $\mleS$ in the limit $n\to\infty$,  and find that the bias of the MLE $\mleS$ appears to converge to positive values only when $|\containA|$ is exponential.  We specifically examine the connected anomaly family $\mathcal{C}_G$ for the graph $G = (V, E)$ whose vertices $V = P \cup C$ are partitioned into two sets: a path graph $P$ and a clique $C$, with $|P \cap C| = 1$. (When $|P| = |C|$, $G$ is known as the ``lollipop graph" \cite{Zhang2009}.) By varying the sizes $|P|, |C|$ of the path graph $P$ and clique $C$, respectively, we can affect the value of $|\containA|$: $|\containA|$ is exponential if $|C| = \Theta(n)$ and is sub-exponential if $|C| = o(n)$.
We observe that $\lim_{n\to\infty} \text{Bias}(|\mleS| /n) > 0$ if $|C| = \Theta(n)$, and $\lim_{n\to\infty} \text{Bias}(|\mleS| /n) = 0$ if $|C| = o(n)$ (Figure \ref{fig:mle}B), which aligns with Conjecture \ref{conj:central_claim}.

In Appendix \ref{sec:append:exps}, we describe two more experiments that support Conjecture \ref{conj:central_claim}.
In the first experiment, we construct an anomaly family $\mathcal{S}$ where $|\mathcal{S}|$ is exponential, while $|\containA|$ is exponential for some anomalies $A$ and sub-exponential for others.
We observe that the MLE $\mleS$ is biased if and only if $|\containA|$ is exponential, providing evidence for Conjecture \ref{conj:central_claim} by showing that $\text{Bias}(|\mleS| /n)$ depends on the number $|\containA|$ of sets containing the anomaly rather than the size $|\mathcal{S}|$ of the anomaly family.
In the second experiment, we empirically demonstrate that the bias of the MLE $\mle{\mathcal{T}_{G}, \rho}$ for the graph cut family $\mathcal{S} = \mathcal{T}_{G, \rho}$ has a strong dependence on the cut-size bound $\rho$. This aligns with Conjecture \ref{conj:central_claim} since $|\containA|$ is polynomial when $\rho$ is constant in $n$  while $|\containA|$ is exponential when $\rho$ is close to the number $|E|$ of edges in $G$ \cite{Nagamochi1994}.

\section{Reducing Bias using Mixture Models}

In the previous section, we showed that the MLE $\mleS$ yields a biased estimate of the size $|A|$ of the anomaly $A$ when 
the number $|\containA|$ of sets in the anomaly family $\mathcal{S}$ that contain $A$ is exponential in $n$. In this section, we derive an anomaly estimator that is less biased than the MLE.  Our anomaly estimator leverages a connection between the ASD and the Gaussian mixture model (GMM), and is motivated by previous work that uses GMMs to estimate \emph{unstructured} anomalies \cite{Cai2007,Donoho2004}.

Recall the following latent variable representation of the ASD: 
given a sample $\mathbf{X} = (X_1, \dots, X_n) \sim \probdist_{\mathcal{S}}(A, \mu)$ from the ASD, we define a corresponding sequence $\mathbf{Z} = (Z_1, \dots, Z_n)$ of latent variables $Z_i = 1(i \in A)$. Estimating the anomaly $A$ is equivalent to estimating the latent variables $\mathbf{Z}$.  The bias of the MLE $\mleS$ corresponds to \emph{overestimating} the sum $|A| = \sum_{i=1}^n Z_i$ of latent variables.

This latent variable representation of the ASD is reminiscent of the latent variable representation of a Gaussian mixture model (GMM), defined as follows.

\begin{namedproblem}[\textbf{Gaussian Mixture Model (2 components, unit variance)}]
Let $\mu > 0$ and $\alpha \in (0,1)$. $X$ is distributed according to the \emph{Gaussian Mixture Model} $\gmm(\mu, \alpha)$ provided
\begin{equation}
    X \sim \alpha N(\mu, 1) + (1-\alpha) N(0,1).
\end{equation}
Associated with $X$ is a latent variable $Z$, where $Z=1$ if $X$ is drawn from the $N(\mu, 1)$ distribution and $Z=0$ if $X$ is drawn from the $N(0, 1)$ distribution.
\end{namedproblem}

Note that $n$ independent observations $X_i \iid \gmm(\mu, \alpha)$ from the GMM are not equal in distribution to a sample $\mathbf{Y} = (Y_1, \dots, Y_n) \sim \probdist_{\mathcal{S}}(A,\mu)$ from the ASD.
In particular, in the GMM all of the data points $X_i$ are identically distributed, while
in the ASD exactly $|A|$ of the data points $Y_i$ are drawn from the $N(\mu,1)$ distribution.
Nevertheless, we observe that the empirical distributions of the unstructured ASD and the GMM converge in Wasserstein distance as $n \to \infty$ (see Appendix \ref{sec:append:wasserstein}).
In anomaly estimation,  some previous approaches model \emph{unstructured} anomalies with a GMM \cite{Cai2007,Donoho2004}.  However, existing work on estimating \emph{structured} anomalies typically models the data with the ASD \cite{Arias-Castro2011,Sharpnack2013}.

Another difference between the ASD and GMM is that 
one can use maximum likelihood estimation to accurately estimate the sum $\sum_{i=1}^n Z_i$ of latent variables $Z_i$ from GMM observations $X_i \iid \text{GMM}(\mu, \alpha)$ \cite{Bishop2006}, unlike with the ASD.
Specifically, \cite{Kalai2010} showed that the following algorithm gives accurate estimates of the  individual latent variables $Z_i$: \textbf{(1)} estimate the GMM parameters $\mu$ and $\alpha$ and \textbf{(2)} set $Z_i = 1$ if the estimated \emph{responsibility} $r_i = P(Z_i = 1 \mid X_i)$, or probability of $X_i$ being drawn from the $N(\mu, 1)$ distribution, is greater than $0.5$. 

In practice, the parameter estimation in step \textbf{(1)} is often done by computing the MLEs $\widehat{\mu}_{\tGMM}$ and $\widehat{\alpha}_{\tGMM}$ of the GMM parameters $\mu$ and $\alpha$, respectively.  For data $X_i \iid \text{GMM}(\mu, \alpha)$ drawn from a GMM, the MLEs $\widehat{\mu}_{\tGMM}$ and $\widehat{\alpha}_{\tGMM}$
are efficiently computed via the EM algorithm \cite{Daskalakis2017,Xu2016} and are asymptotically unbiased estimators of $\mu$ and $\alpha$, respectively \cite{Chen2017}.


Motivated by the connection between the latent variable representations of the ASD and GMM, we prove an analogous result on the asymptotic unbiasedness of the MLEs $\widehat{\mu}_{\tGMM}$ and $\widehat{\alpha}_{\tGMM}$ for data drawn from the ASD. 
Specifically, we prove that given data $\mathbf{X} \sim \probdist_{\mathcal{S}}(A, \mu)$ with sufficiently large mean $\mu$, then the GMM MLEs $\widehat{\mu}_{\tGMM}$ and $\widehat{\alpha}_{\tGMM}$ obtained by fitting a GMM to data $\mathbf{X}$ are asymptotically unbiased estimators of $\mu$ and $|A|/n$, respectively.  This result settles a conjecture of \cite{NetMix}.

\begin{theorem}
\label{thm:gmm_likelihood}
Let $\mathbf{X} = (X_1, \dots, X_n) \sim \asdS$, where $|A|=\alpha n$ for $0 < \alpha < 0.5$ and $\mu \geq C\sqrt{\log{n}}$ for a sufficiently large constant $C > 0$.  For sufficiently large $n$,  we have that $|\alphaGMM - \alpha| \leq \sqrt{\frac{\log{n}}{n}}$ and $|\muGMM - \mu| \leq 3\sqrt{\frac{\log{n}}{n}}$
with probability at least $ 1-\frac{1}{n}$.
\end{theorem}

A sketch of our proof of Theorem \ref{thm:gmm_likelihood} is as follows.  Let $B = \left\{(\widehat{\alpha}, \widehat{\mu}) : \text{$|\alpha| > \sqrt{\frac{\log{n}}{n}}$ or $|\widehat{\mu} -\mu| > \sqrt{\frac{\log{n}}{n}}$} \right\}$ be the set of all ``bad" estimators $(\widehat{\alpha}, \widehat{\mu})$ of the true GMM parameters $(\alpha, \mu)$. We show that with high probability, the GMM likelihood for all $(\widehat{\alpha}, \widehat{\mu}) \in B$ is less than the GMM likelihood for $(\alpha, \mu)$, which implies that the GMM MLE $(\widehat{\alpha}_{\text{GMM}}, \widehat{\mu}_{\text{GMM}})$ is not in $B$.

\begin{figure*}[t]
    \centering
    \includegraphics[width=\linewidth]{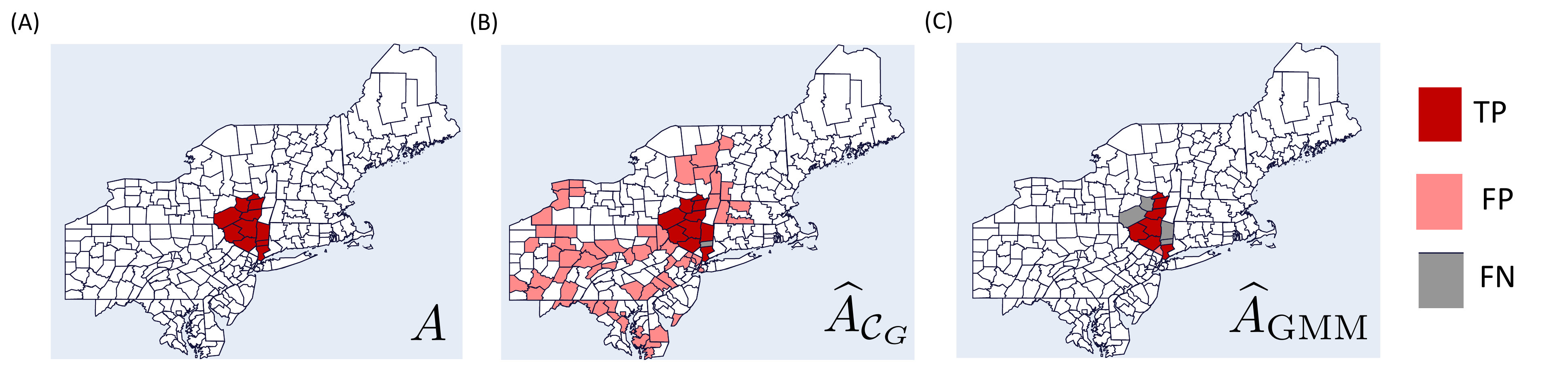}
    \caption{
(A) An anomaly $A$ containing $11$ connected counties is implanted into a graph of counties in the Northeast USA \cite{Cadena2019}.  (B)  The MLE $\mle{\mathcal{C}_G}$ greatly overestimates the size of the anomaly with $59$ false positives ($F$-measure = $0.24$).
    (C) The GMM estimator $\widehat{A}_{\tGMM}$ identifies 7/11 counties correctly with only 1 false positive ($F$-measure = $0.73$).}
    \label{fig:compare_gmm_asd}
\end{figure*}

\subsection{A GMM-based Anomaly Estimator}
Motivated by Theorem \ref{thm:gmm_likelihood}, we use a GMM fit to derive an asymptotically unbiased anomaly estimator for any anomaly family $\mathcal{S}$.
Our approach
generalizes the algorithm given in \cite{NetMix} for the connected family $\mathcal{S} = \mathcal{C}_G$. 
Our approach is inspired by both the GMM literature discussed above and by classical statistical techniques such as the False Discovery Rate (FDR) \cite{Benjamini1995} and the Higher Criticism \cite{Donoho2004} thresholding procedures, which identify unstructured anomalies in $z$-score distributions by first estimating the size of the anomalies \cite{Jin2007,Cai2007,Meinshausen2006,Benjamini2010,Vinayak2020}.

Given data $\mathbf{X} \sim \probdist_{\mathcal{S}}(A, \mu)$, we first use the EM algorithm to fit a GMM to the data $\mathbf{X}$. This fit yields estimates $\widehat{\mu}_{\text{GMM}}, \widehat{\alpha}_{\text{GMM}}$ of the GMM parameters $\mu, \alpha$, respectively, as well as estimates $\widehat{r}_i$ of the responsibilities $r_i = P(Z_i = 1 \mid X_i)$.  Our estimator $\widehat{A}_{\text{GMM}}$ is the set $S \in \mathcal{S}$ with size
$\big| |S| - \widehat{\alpha}_{\text{GMM}} n  \big| \leq \sqrt{\frac{\log{n}}{n}}$
and having the largest total responsibility:
\begin{equation}
\label{eq:anomaly_mix}
\widehat{A}_{\text{GMM}} = \argmax_{\substack{S \in \mathcal{S} \\  \big| |S| - \widehat{\alpha}_{\text{GMM}} n  \big| \leq \sqrt{\frac{\log{n}}{n}}  }}  \left( \sum_{i \in S}  \widehat{r}_i \right).
\end{equation}
By Theorem \ref{thm:gmm_likelihood}, our constraint on the size $|S|$  in \eqref{eq:anomaly_mix} ensures that the size $|\widehat{A}_{\text{GMM}}|$ of the GMM-based estimator $\widehat{A}_{\text{GMM}}$ has asymptotically zero bias for sufficiently large $\mu$.  We formalize this in the following Corollary.

\begin{cor}
\label{cor:gmm1}
Let $\mathbf{X} = (X_1, \dots, X_n) \sim \asdS$, where $|A|=\alpha n$ for $0 < \alpha < 0.5$ and $\mu \geq C\sqrt{\log{n}}$ for a sufficiently large constant $C > 0$. Then $\lim_{n\to\infty} \textnormal{Bias}(|\mleS|/n) = 0$.
\end{cor}

In addition, for the unstructured family $\mathcal{S} = \mathcal{P}_n$,  we show our estimator $\widehat{A}_{\tGMM}$ has small normalized error $\frac{|A \triangle \widehat{A}_{\text{GMM}}|}{|A|}$, as studied by \cite{Castro2014,Castro2017}, where $\triangle$ is the symmetric set difference.

\begin{cor}
\label{cor:gmm2}
Let $\mathbf{X} = (X_1, \dots, X_n) \sim \text{ASD}_{\mathcal{P}_n}(A, \mu)$, where $|A|=\alpha n$ for $0 < \alpha < 0.5$ and $\mu \geq C \sqrt{\log{n}}$ for a sufficiently large constant $C> 0$. Then $\frac{|A \triangle \widehat{A}_{\text{GMM}}|}{|A|} \leq 2 \sqrt{\frac{\log{n}}{n}} = o(1)$ with probability at least $1-\frac{1}{n}$.
\end{cor}

Another useful property of our estimator $\widehat{A}_{\text{GMM}}$ is that the objective $\sum_{i\in S} \widehat{r}_i$ in \eqref{eq:anomaly_mix} is linear, in contrast to the non-linear objective $\sum_{i \in S} X_i /\sqrt{|S|}$ for the MLE $\mleS$ in Equation \eqref{eq:mle_def}. Thus, our estimator $\widehat{A}_{\text{GMM}}$ can be efficiently computed for many anomaly families $\mathcal{S}$.
For the unstructured family $\mathcal{P}_n$, $\widehat{A}_{\text{GMM}}$ can be computed in $O(n \log{n})$ time by sorting the data points and returning the $\lfloor \widehat{\alpha}_{\tGMM} n + \sqrt{\log{n}/n}  \rfloor$ largest ones. For the interval family $ \mathcal{I}_n$, $\widehat{A}_{\text{GMM}}$ can be computed in $O(n)$ time by scanning over all intervals of size $\lfloor \widehat{\alpha}_{\tGMM} n + \sqrt{\log{n}/n}  \rfloor$. For the graph cut family $\mathcal{T}_{G,\rho}$, \cite{Sharpnack2013b} shows that \eqref{eq:anomaly_mix} can be efficiently solved with a convex program through the use of Lov\'{a}sz extensions \cite{Bach2010}.

More generally, when the constraint $S \in \mathcal{S}$ can be expressed with linear constraints, one can compute $\widehat{A}_{\text{GMM}}$ with an Integer Linear Program (ILP). This is true for anomaly families including the submatrix family $\mathcal{M}_N$, the graph cut family $\mathcal{T}_{G,\rho}$ \cite{Sharpnack2013}, and the connected family $\mathcal{C}_G$ \cite{Dittrich2008, NetMix}. In practice, we found that directly computing \eqref{eq:anomaly_mix} via ILP could sometimes be inefficient for the submatrix and connected families, and in Appendix \ref{sec:append:gmm_approx} we derive an approximation to \eqref{eq:anomaly_mix} that can be efficiently computed for these families. 

\subsection{Experiments}

First, we compare the performance of our estimator $\widehat{A}_{\text{GMM}}$  to the MLE 
$\mleS$ for the anomaly families $\mathcal{S}$ from Section \ref{subsec:exp_evidence}.
We observe that $\text{Bias}(|\widehat{A}_{\tGMM}| / n) \approx 0$ for all means $\mu \geq \mu_{\text{detect}}$ and across many anomaly families $\mathcal{S}$ (Figure \ref{fig:mle}C). We also observe that $\lim_{n\to\infty} \text{Bias}(|\widehat{A}_{\tGMM}| / n)  = 0$ no matter if $|\containA|$ is exponential or sub-exponential (Figure \ref{fig:mle}D). This empirically demonstrates Theorem \ref{thm:gmm_likelihood} by showing that $|\widehat{A}_{\tGMM}|$ is an asymptotically unbiased estimator of the anomaly size $|A|$ for sufficiently large $\mu$ regardless of the number $|\containA|$ of sets containing the anomaly $A$.

Next, we simulate a disease outbreak on 
the Northeastern USA Benchmark (NEast) graph, a standard benchmark for estimating spatial anomalies \cite{Cadena2018,Cadena2019}. 
The NEast graph $G = (V, E)$ is a graph whose nodes are the $n=244$ counties in the northeastern part of the USA \cite{Kulldorff2003} with edges connecting adjacent counties. 
Similar to \cite{Cadena2018,Aksoylar2017, Qian2014}, we implant a connected anomaly $A \in \mathcal{C}_G$ of size $|A|=11$ and we draw a sample $\mathbf{X} \sim \probdist_{\mathcal{C}_G}(A, 2)$.
Because existing methods for estimating anomalous subgraphs typically compute the MLE $\mle{\mathcal{C}_G}$ \cite{Chen2014,Qian2014,Cadena2019},  we also compare our estimator to the MLE.
We find (Figure \ref{fig:compare_gmm_asd})  that the MLE $\mle{\mathcal{C}_G}$ greatly overestimates the size $|A|$ of the anomaly $A$, with many more false positives compared to the GMM estimator $\widehat{A}_{\tGMM}$. 

We also compare our estimator $\widehat{A}_{\text{GMM}}$ and the MLE $\mleS$ on a real-world highway traffic dataset; similar to the NEast graph, this dataset is also often studied in the scan statistic literature \cite{Zhou2016,Cadena2018,Cadena2019}.  This dataset consists of a highway traffic network $G=(V,E)$ in Los Angeles County, CA with $|V|=1868$ vertices and $|E|=1993$ edges. The vertices $V$ are sensors that record the speed of cars passing and the edges $E$ connect adjacent sensors. The observations $\mathbf{X} = (X_v)_{v \in V}$ are $p$-values (where sensors that record higher average speeds have lower $p$-values) that are transformed to Gaussians using the method in \cite{NetMix}. 

For the connected family $\mathcal{C}_G$, we find that our estimator $\widehat{A}_{\text{GMM}}$ is much smaller than the MLE $\mle{\mathcal{C}_G}$ ($|\widehat{A}_{\text{GMM}}| = 17$ versus $|\mle{\mathcal{C}_G}| = 140$) but with higher average score 
($2.6$ for our estimator versus $0.55$ for the MLE). 
While there is 
no ground-truth anomaly in this dataset, our results show that our estimator $\widehat{A}_{\text{GMM}}$ yields a smaller anomaly but with higher average values than the MLE $\mle{\mathcal{C}_G}$,
consistent with the theoretical results in Section \ref{sec:mle_bias} that the MLE $\mle{\mathcal{C}_G}$ is a biased estimator.  In Appendix \ref{sec:append:exps}, we show  similar results comparing our estimator $\widehat{A}_{\text{GMM}}$ and the MLE $\mleS$ for the edge-dense family $\mathcal{S}=\mathcal{E}_{G, 0.7}$. Since the goal of anomaly estimation in this application is to identify portions of roads with or without high traffic volume, the large and biased anomaly estimates produced by the MLE may not be useful for traffic studies. 

\begin{wrapfigure}{l}{0.5\textwidth}
  \begin{center}
    \includegraphics[width=\linewidth]{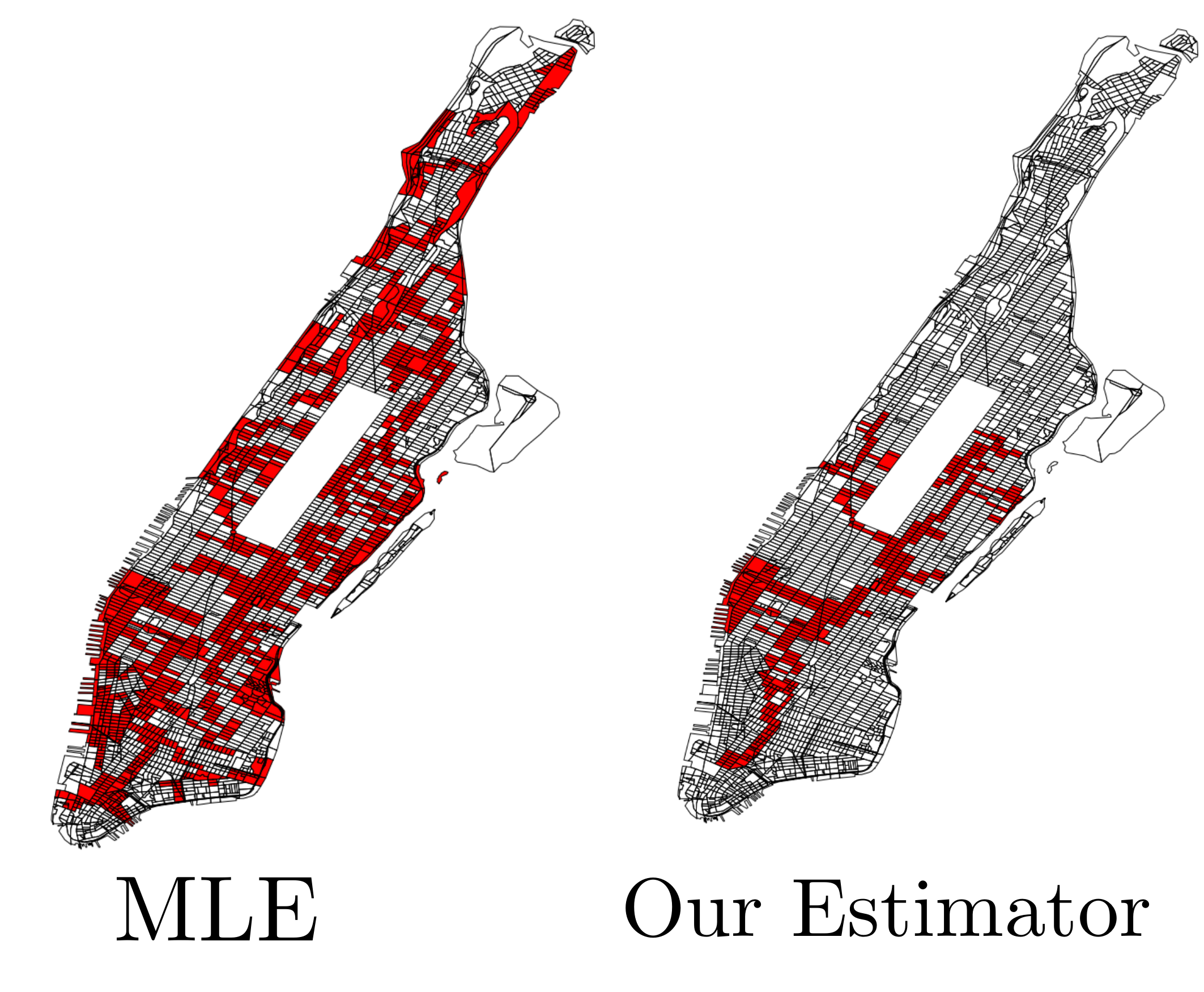}
  \end{center}
\caption{Comparison of the MLE,  also known as the graph scan statistic (left), and our estimator (right) in estimating connected disease clusters in data of breast cancer incidence in Manhattan.  Our estimator computes a smaller cluster than the MLE/graph scan statistic but with a 20\% higher average incidence ratio (relative risk). \vspace{-0.5cm}}
\label{fig:manhattan}
\end{wrapfigure}

We also compared our estimator and the MLE on a dataset of breast cancer incidence in census blocks in Manhattan \cite{Boscoe2016} using the connected family $\mathcal{C}_G$. 
This dataset is typically modeled with Poisson distributions,  and we accordingly adapted the MLE and our estimator to such Poisson distributions (see Appendix \ref{sec:append:exps} for more details). In this setting, the MLE is also known as a graph scan statistic \cite{Cadena2019}.  We find that our estimator identifies a much smaller connected cluster of breast cancer cases compared to the MLE/graph scan statistic (182 census blocks vs 382, Figure \ref{fig:manhattan}) but with  a 20\% higher cancer incidence rate, again demonstrating the bias of the MLE.

\section{Conclusion}

We study the problem of estimating structured anomalies. We formulate this problem as the problem of estimating a parameter of the Anomalous Subset Distribution (ASD), with the structure of the anomaly described by an anomaly family. We demonstrate that the Maximum Likelihood Estimator (MLE) of the size of this parameter is biased if and only if 
the number of sets in the anomaly family containing the anomaly is exponential. These results unify existing results for specific anomaly families including intervals, submatrices, and connected subgraphs. Next, we develop an asymptotically unbiased estimator using a Gaussian mixture model (GMM), and 
empirically demonstrate the advantages of our estimator on both simulated and real datasets.

Our work opens up a number of future directions. First, it would be highly desirable to provide a complete proof of Conjecture \ref{conj:central_claim}.
A second direction is to generalize the ASD to more than one anomaly in a dataset by building on existing work for the interval family \cite{Jeng2010} and the submatrix family \cite{Chen2016}. One potential algorithm for identifying multiple anomalies is to fit a $k$-component GMM to the data and sequentially compute each anomaly.
A third direction is to generalize our theoretical results to other distributions, e.g.
Poisson distributions, which are commonly used to model anomalies in integer-valued data \cite{Cadena2019, Liu2019, Kulldorff1997}.  While our GMM-based estimator is easily adapted to other distributions, one challenge in studying bias is that the MLE does not necessarily have a simple form like it does for Gaussian distributions.
These directions would strengthen the theoretical foundations for further applications of anomaly estimation.

\section*{Acknowledgments}
The authors would like to thank Allan Sly for helpful discussions, and Baojian Zhou and Martin Zhu for assistance with running the Graph-GHTP code \cite{Zhou2016}. 
U.C. is supported by NSF GRFP DGE 2039656.
J.C.H.L. is partially supported by NSF award IIS-1562657.
B.J.R. is supported by a US National Institutes of Health (NIH) grant U24CA211000.

\bibliography{references}
\bibliographystyle{ieeetr}

\newpage

\appendix

\renewcommand{\thefigure}{S\arabic{figure}}
\setcounter{figure}{0} 

\section{Calculating $\mu_{\text{detect}}$}
\label{sec:append:mudetect}

$\mu_{\text{detect}}$ is the smallest mean $\mu$ such that
the GLR test asymptotically solves the ASD Detection Problem with the probability of a type 1 or type 2 error going to $0$ as $n \to \infty$ \cite{Sharpnack2013b}. We empirically determine $\mu_{\text{detect}}$ by finding the smallest mean $\mu$ such that the Type I and Type II errors of the GLR test statistic $\widehat{t}_{\mathcal{S}}$ (Equation \eqref{eq:glr_test_stat}) are both less than $0.01$.

\section{Additional Experiments}
\label{sec:append:exps}

\begin{figure*}[t]
    \centering
\includegraphics[scale=0.43]{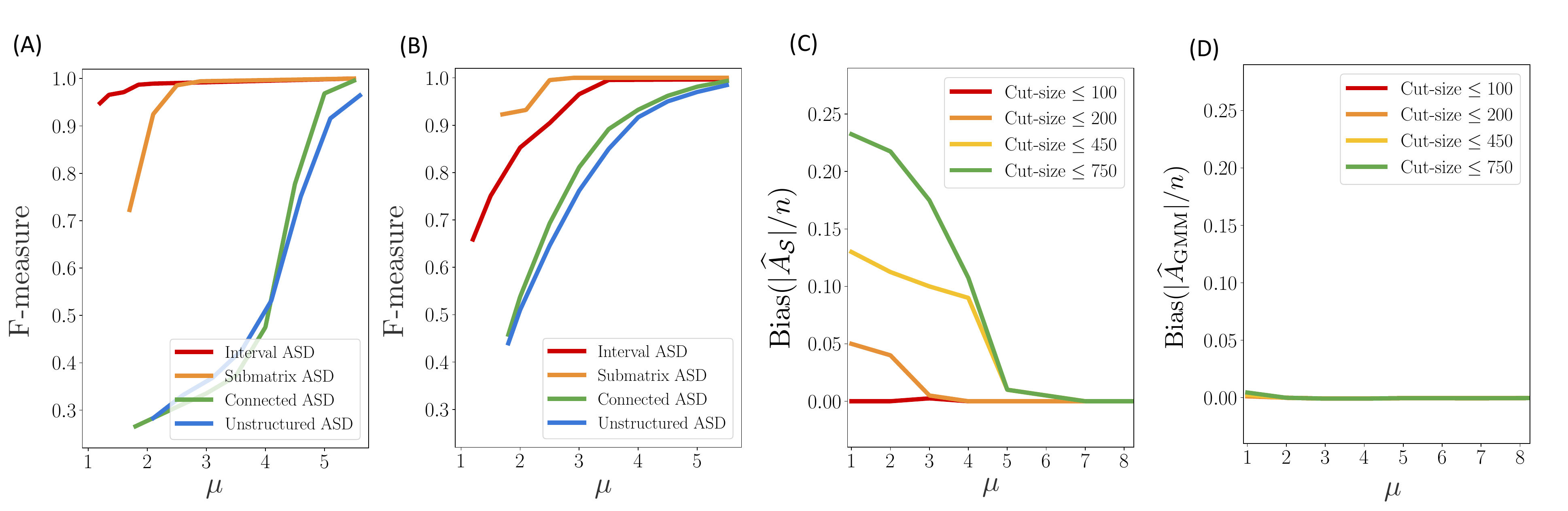}
\caption{Data $\mathbf{X} = (X_1, \dots, X_n) \sim \probdist_{\mathcal{S}}(A, \mu)$ as in Figure \ref{fig:mle}. (A) F-measure of the MLE $\mleS$ versus $\mu$ for means $\mu \geq \mu_{\text{detect}}$.  Note that the MLEs $\mleS$ for the connected family $\mathcal{S} = \mathcal{C}_G$ and unstructured family $\mathcal{S} = \mathcal{P}_n$ have low F-measure for small means $\mu$, consistent with the $\text{Bias}(|\mleS|/n)$ for these families shown in Figure 2.
(B) F-measure of our estimator $\widehat{A}_{\tGMM}$ versus $\mu$ for means $\mu \geq \mu_{\text{detect}}$. (C) $\text{Bias}(|\mleS|/n)$ of the MLE versus $\mu$ for means $\mu \geq 1$ and for graph cut family $\mathcal{S} = \mathcal{T}_{G, \rho}$ with different bounds $\rho$ on the cut-size and.  (D) $\text{Bias}(|\widehat{A}_{\tGMM}|/n)$ of our GMM estimator versus $\mu$ for means $\mu \geq 1$ and for graph cut family $\mathcal{S} = \mathcal{T}_{G, \rho}$ with different bounds $\rho$ on the cut-size.
}
    \label{fig:mle_fmeas_graph}
\end{figure*}

\subsection{F-measure}

Although Conjecture \ref{conj:central_claim} is about the $\text{Bias}(|\mleS|/n)$ of the MLE $\mleS$, we also observe that larger $\text{Bias}(|\mleS|/n)$ reduces the F-measure between the anomaly $A$ and the MLE $\mleS$.  Using the data described in Section \ref{subsec:exp_evidence}, we find a noticeable difference in F-measure between anomaly families where $|\containA|$ is exponential --- the connected family $\mathcal{C}_G$ and the unstructured family $\mathcal{P}_n$ --- and anomaly families where $|\containA|$ is sub-exponential --- the interval family $\mathcal{I}_n$ and the submatrix family $\mathcal{M}_N$ (Figure \ref{fig:mle_fmeas_graph} A).

In contrast, our GMM-based estimator $\widehat{A}_{\tGMM}$ has a much smaller difference in the F-measure for anomaly families where $|\containA|$ is exponential versus anomaly families where $|\containA|$ is sub-exponential (Figure \ref{fig:mle_fmeas_graph} B). This result is consistent with the reduced bias of the GMM-based estimator $\widehat{A}_{\tGMM}$ (Figure \ref{fig:mle}C). Interestingly, even for our reduced bias estimator, we still observe a mild difference in F-measure between the families with exponential $|\containA|$ versus the families with sub-exponential $|\containA|$.


\subsection{Graph Cut Family}

We examine the $\text{Bias}(|\mle{\mathcal{T}_{G, \rho}}|/n)$ of the size of the MLE $\mle{\mathcal{T}_{G, \rho}}$ for the graph cut family $\mathcal{T}_{G, \rho}$, where $G$ is a $\sqrt{n}\times\sqrt{n}$ lattice graph, for different values of the bound $\rho$ on the cut-size. For each value of $\rho$, we select an anomaly $A \in \mathcal{T}_{G, \rho}$ with size $|A| = 0.05n$ uniformly at random from $\mathcal{T}_{G, \rho}$. (Note that the cut-size of $A$ is not fixed, as we select $A$ uniformly at random from the set $\mathcal{T}_{G, \rho}$ of all subgraphs of $G$ with cut-size less than $\rho$.)
We then draw a sample $\mathbf{X} = (X_1, \dots, X_n) \sim \probdist_{\mathcal{T}_{G, \rho}}(A, \mu)$ with $n=900$ observations and compute the MLE $\mle{\mathcal{T}_{G, \rho}}$.
We repeat for $50$ samples to estimate $\text{Bias}(|\mle{\mathcal{T}_{G, \rho}}|/n)$. 

While the graph cut anomaly family is often studied in the network anomaly literature \cite{Sharpnack2013,Sharpnack2013b,Sharpnack2013c}, the cut-size bound $\rho$ is typically left unspecified.
When $\rho$ is constant $|\containA|$ is polynomial in $n$, but when $\rho$ is close to the number of edges in $G$ then $|\containA|$ is exponential in $n$ \cite{Nagamochi1994}. So by Conjecture \ref{conj:central_claim} we expect the bias of the MLE $\mle{\mathcal{T}_{G, \rho}}$ to depend on $\rho$.
Indeed, we observe that the $\text{Bias}(|\mle{\mathcal{T}_{G, \rho}}| /n)$ of the MLE is small when $\rho$ is small and the $\text{Bias}(|\mle{\mathcal{T}_{G, \rho}}| /n)$ of the MLE is large when $\rho$ is large (Figure \ref{fig:mle_fmeas_graph} C), which is consistent with Conjecture \ref{conj:central_claim}. Our results demonstrate that careful attention to the cut-size bound $\rho$ is required when the MLE $\mle{\mathcal{T}_{G, \rho}}$ is used for anomaly estimation.

For the same data, we find that our GMM estimator $\widehat{A}_{\tGMM}$ has small bias regardless of the cut-size bound $\rho$ (Figure \ref{fig:mle_fmeas_graph} D). This is consistent with Corollary \ref{cor:gmm1}, and demonstrates that our GMM estimator $\widehat{A}_{\tGMM}$ is a less biased estimator than the MLE $\mle{\mathcal{T}_{G}, \rho}$ regardless of the cut-size bound $\rho$.

\subsection{Dependence of $\text{Bias}(|\mleS|/n)$ on $|\containA|$ versus $|\mathcal{S}|$}

\begin{figure}[t]
    \centering
    \includegraphics[width=0.48\linewidth]{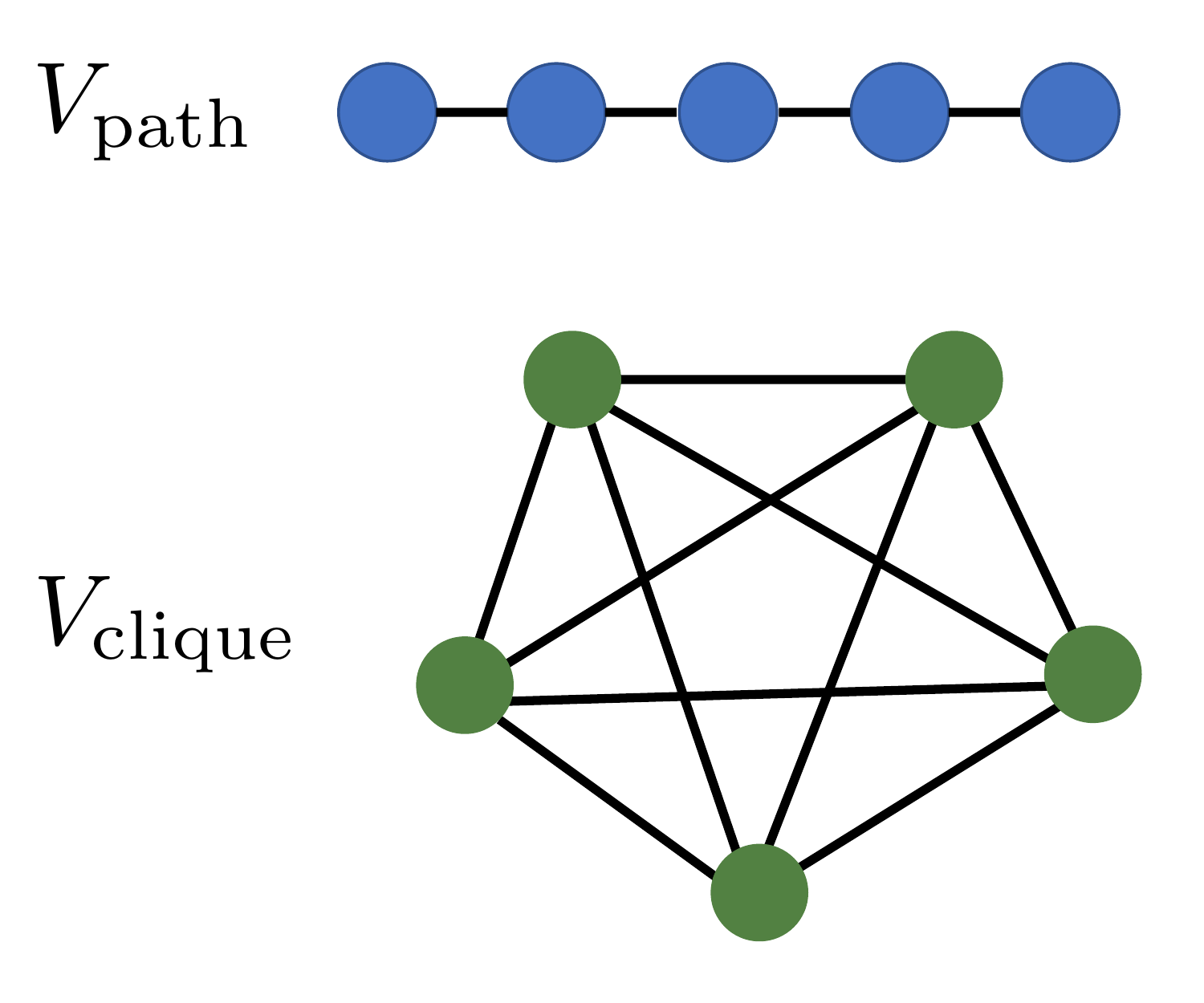}
    \includegraphics[width=0.48\linewidth]{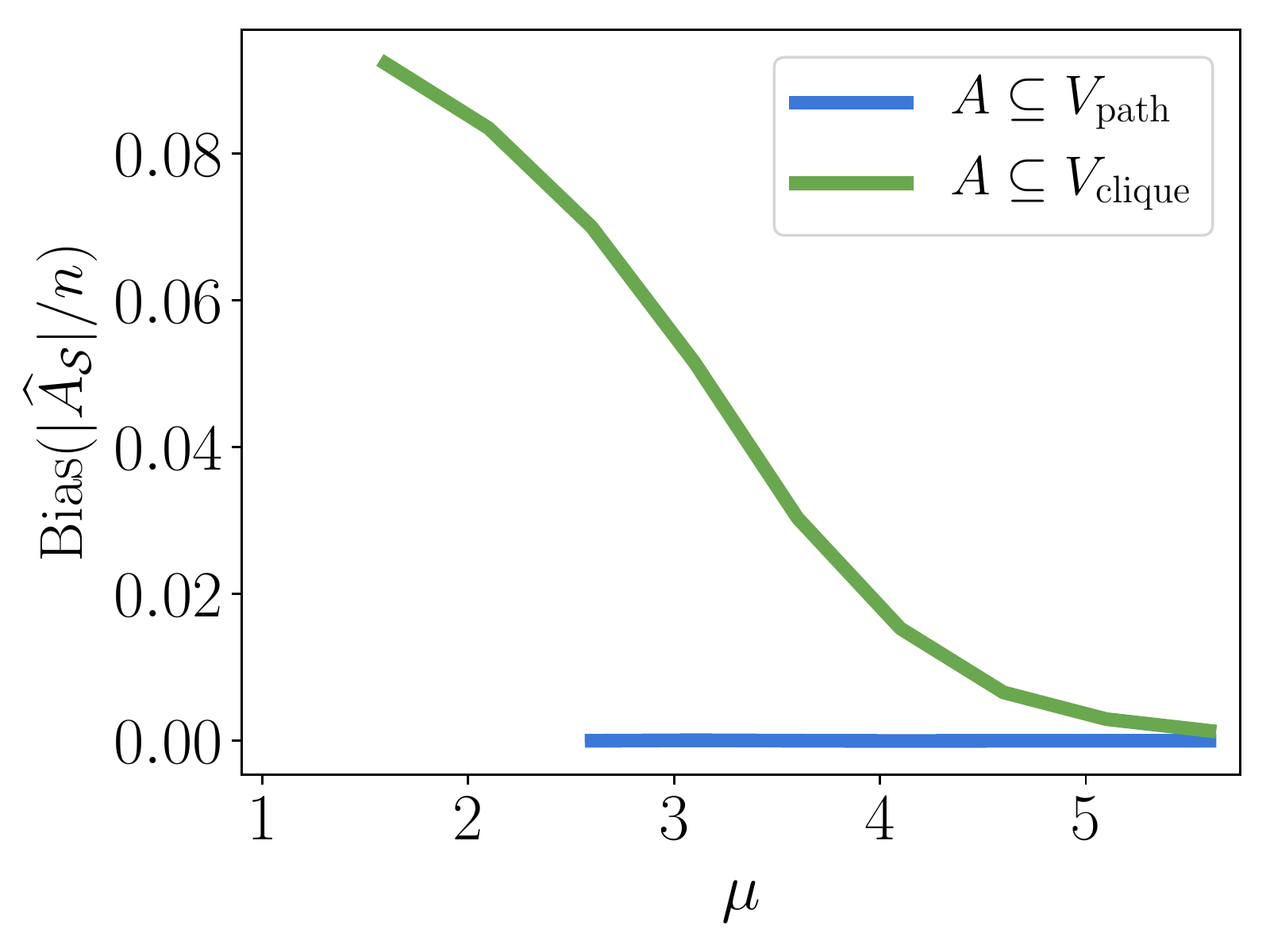}
    \caption{\textbf{Left:} Graph $G=(V, E)$ with two disjoint connected components: $V_{\text{path}}$, a path graph, and $V_{\text{clique}}$, a clique graph, with $|V_{\text{path}}| = |V_{\text{clique}}| = \frac{n}{2}$. \textbf{Right:} $\text{Bias}(|\mleS|/n)$ versus mean $\mu$ for the connected anomaly family $\mathcal{S} = \mathcal{C}_{G}$ with $n=|V| = 500$ vertices and an anomaly $A$ with size $|A| = 0.05n$, for means $\mu \geq \mu_{\text{detect}}$. The blue line corresponds to an anomaly $A \subseteq V_{\text{path}}$ and the green line corresponds to an anomaly $A \subseteq V_{\text{clique}}$.  This experiment suggests that the $\text{Bias}(|\mleS|/n)$ is determined by the $|\containA|$, rather than $|\mathcal{S}|$, consistent with Conjecture \ref{conj:central_claim}.}
    \label{fig:path_clique}
\end{figure}

In this section, we construct an anomaly family $\mathcal{S}$ where $|\mathcal{S}|$ is exponential, but $|\containA|$ is exponential for some anomalies $A$ and sub-exponential for others. We then use this anomaly family $\mathcal{S}$ to provide evidence that $\text{Bias}(|\mleS|/n)$ depends on the number $|\containA|$ of subsets in $\mathcal{S}$ that contain the anomaly $A$, rather than the size $|\mathcal{S}|$ of the anomaly family.

Let $G=(V, E)$ be a graph whose vertices $V= V_{\text{path}} \cup V_{\text{clique}}$ can be partitioned into two disjoint connected components: $V_{\text{path}}$, a path graph, and $V_{\text{clique}}$, a clique (Figure \ref{fig:path_clique}, left).  (Note that the path graph $V_{\text{path}}$ and the clique $V_{\text{clique}}$ are disjoint, unlike the graph from Figure \ref{fig:mle}.) Both the path graph $V_{\text{path}}$ and the clique $V_{\text{clique}}$ have size $|V_{\text{path}}| = |V_{\text{path}}| = \frac{n}{2}$, where $n=900$.

Let $\mathcal{S} = \mathcal{C}_G$ be the connected family for graph $G$, and let $A \in \mathcal{C}_G$ be a set of size $|A| = 0.05 n$.
The size $|\mathcal{S}|$ of the anomaly family $\mathcal{S}$ is exponential in $n$, as $|\mathcal{S}| = O(2^{\frac{n}{2}})$. 
However, $|\containA|$ depends on the anomaly $A$: if the anomaly $A \subseteq V_{\text{path}}$ is in the path graph component, then $|\containA| = O(n^2)$ is sub-exponential in $n$. On the other hand, if $A \subseteq V_{\text{clique}}$ is in the clique graph component, then $|\containA| = O(2^{0.45n})$ is exponential in $n$. 

Empirically, we observe that  if $\mu \geq \mu_{\text{detect}}$, then $\text{Bias}(|\mleS|/n) \approx 0$ if $A \subseteq V_{\text{path}}$ and $\text{Bias}(|\mleS|/n) > 0$ if $A \subseteq V_{\text{clique}}$ (Figure \ref{fig:path_clique}, right). This finding is consistent with Conjecture \ref{conj:central_claim} and demonstrates the dependence of $\text{Bias}(|\mleS|/n)$ on $|\containA|$ rather than $|\mathcal{S}|$.

\subsection{Highway Traffic Data with Edge-Dense Family}

We compare our estimator $\widehat{A}_{\text{GMM}}$ and the MLE $\mleS$ on a real-world highway traffic dataset. This dataset consists of a highway traffic network $G=(V,E)$ in Los Angeles County, CA with $|V|=1868$ vertices and $|E|=1993$ edges. The vertices $V$ are sensors that record the speed of cars passing and the edges $E$ connect adjacent sensors. The observations $\mathbf{X} = (X_v)_{v \in V}$ are $p$-values (where sensors that record higher average speeds have lower $p$-values) that are transformed to Gaussians using the method in \cite{NetMix}. 

For the edge-dense family $\mathcal{E}_{G,\delta}$ with edge density $\delta=0.7$, we find that our GMM-based estimator $\widehat{A}_{\text{GMM}}$ is much smaller than the MLE $\mle{\mathcal{E}_{G,\delta}}$ ($|\widehat{A}_{\text{GMM}}| = 10$ versus $|\mle{\mathcal{E}_{G,\delta}}| = 600$) but with higher average score 
($4.4$ for our estimator versus $0.4$ for the MLE). 
While there is 
no ground-truth anomaly in this dataset, our results show that our estimator $\widehat{A}_{\text{GMM}}$ yields a smaller anomaly but with higher average values than the MLE $\mle{\mathcal{E}_{G,\delta}}$, which also suggests that the MLE $\mle{\mathcal{E}_{G,\delta}}$ for the edge-dense family $\mathcal{E}_{G,\delta}$ is biased.

\subsection{Additional Details for NYC Breast Cancer}

In the NYC breast cancer incidence data \cite{Boscoe2016}, we are given observed disease counts $\mathbf{C}=\{C_1, \dots, C_n\}$ and expected disease counts $\mathbf{B}=\{B_1, \dots, B_n\}$ for each census block $i\in [n]$.  As is standard in the disease surveillance and spatial scan statistic literature, \cite{Kulldorff1997,Glaz2010,Neill2009,Neill2012},  we model the distributed of the observed counts $\mathbf{C}$ as
\begin{equation}
\label{eq:poisson_scan_stat}
    C_i \sim 
    \begin{cases}
    \text{Pois}(q_{\text{in}} B_i) & \text{if } i \in A,\\
    \text{Pois}(B_i) & \text{otherwise},
    \end{cases}
\end{equation}
where $A \in \mathcal{S}$ is the anomaly, $\mathcal{S}$ is the anomaly family, and $q_{\text{in}}$ is the \emph{relative risk} of census blocks $i \in A$ in the anomaly $A$.  

The MLE $\widehat{A}_{\mathcal{S}}$ for the anomaly $A$ given the observed counts $\mathbf{C}$ and the expected counts $\mathbf{B}$ --- also known as the expectation-based Poisson scan statistic \cite{Neill2009} or the graph scan statistic \cite{Cadena2019} --- is given by
\begin{equation}
\widehat{A}_{\mathcal{S}} = \argmax_{A \in \mathcal{S}} \left[ \sum_{i \in A} (B_i) + \left( \sum_{i \in A} C_i \right) \cdot \left( -1 + \log\sum_{i \in A} C_i - \log\sum_{i \in A} B_i \right) \right].
\end{equation}

\begin{figure}[t]
    \centering
    \includegraphics[width=0.48\linewidth]{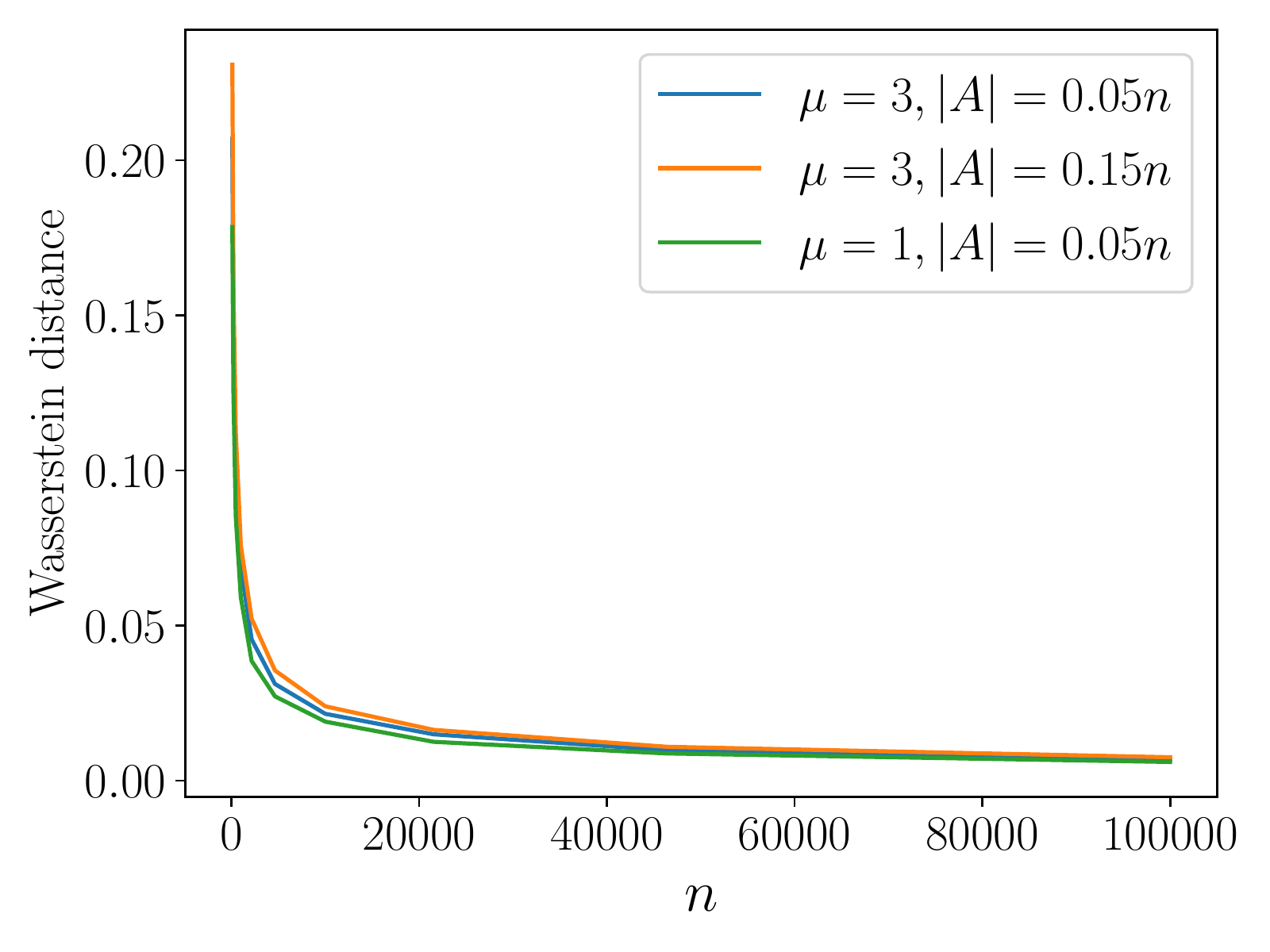}
    \includegraphics[width=0.48\linewidth]{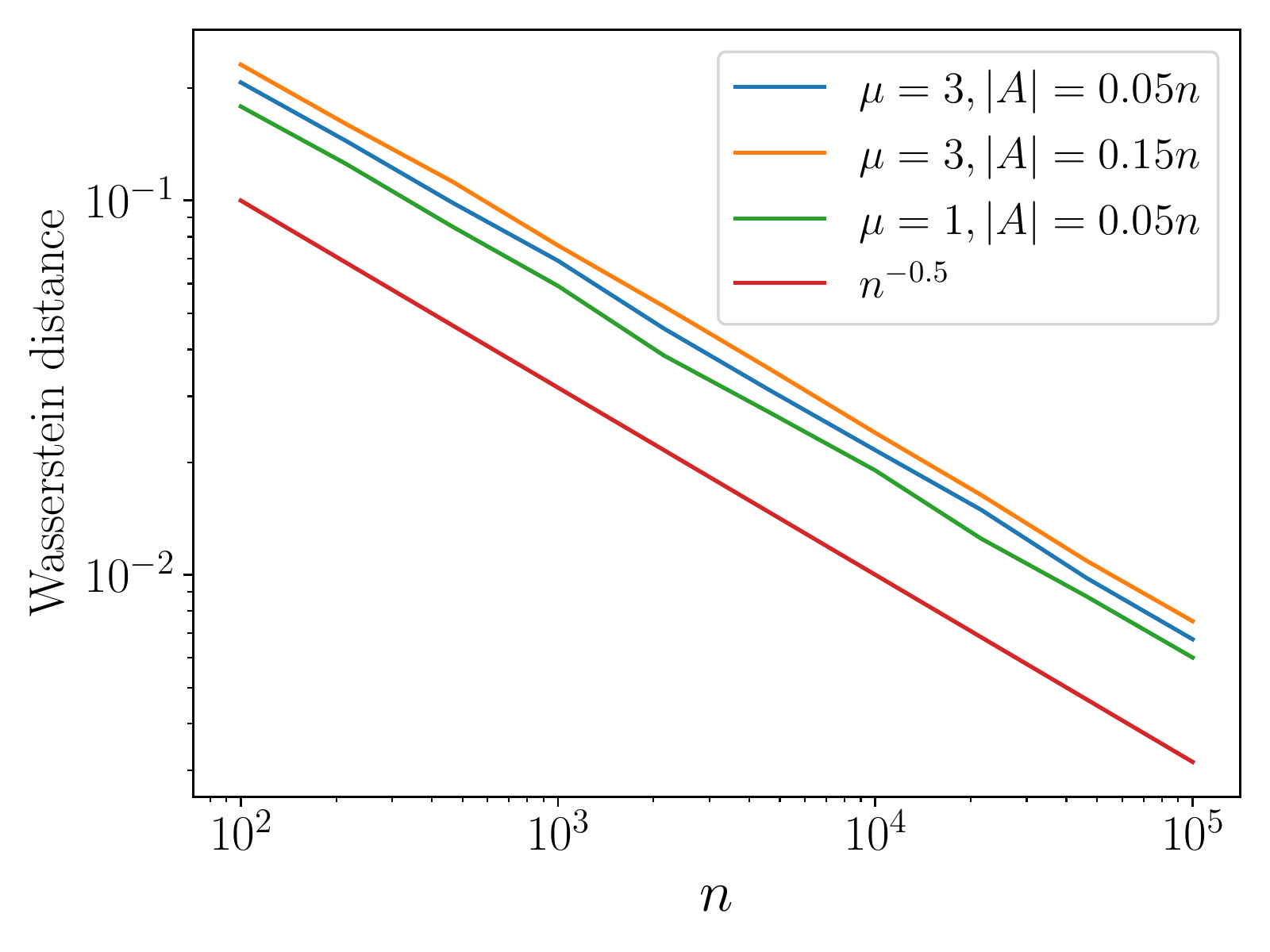}
    \caption{$1$-Wasserstein distance between the GMM distribution and the unstructured ASD distribution.  $\mathbf{X} = (X_1, \dots, X_n) \iid \gmm(\mu, \alpha)$ is distributed according to the GMM 
    and $\mathbf{Y} = (Y_1, \dots, Y_n) \sim \probdist_{\mathcal{P}_n}(A,\mu)$ is distributed according to the unstructured ASD, with $\alpha =  |A| / n$. \textbf{Left:} $d_W\left( \frac{1}{n} \sum_{i=1}^n 1_{X_i}, \frac{1}{n} \sum_{i=1}^n 1_{Y_i} \right)$, the $1$-Wasserstein distance between the empirical distributions of the GMM and unstructured ASD, versus the number $n$ of observations for various values of $\mu$ and $|A|/n$. \textbf{Right:} 
    $1$-Wasserstein distance on log-log scale.
    We observe that the $1$-Wasserstein distance $d_W\left( \frac{1}{n} \sum_{i=1}^n 1_{X_i}, \frac{1}{n} \sum_{i=1}^n 1_{Y_i} \right)$ is  $O(n^{-0.5})$, as each line is parallel to $n^{-0.5}$ in the log-log plot.}
    \label{fig:wasserstein}
\end{figure}

We adapt our estimator to the disease count model in Equation \eqref{eq:poisson_scan_stat} by using the EM algorithm to fit the observed counts $\mathbf{C}$ to the Poisson mixture $C_i \sim \alpha \cdot \text{Pois}(q_{\text{in}} B_i) + (1-\alpha) \cdot \text{Pois}(B_i)$.  The rest of our estimator is unchanged: we use the observed count fit to compute the responsibilities $\widehat{r}_i = P(i \in A \mid C_i, B_i)$ for each census block $i \in [n]$ and then estimate the anomaly using Equation \eqref{eq:anomaly_mix}.

\section{Wasserstein Distance between GMM and Unstructured ASD}
\label{sec:append:wasserstein}
Let $\mathbf{X} = (X_1, \dots, X_n)$ with $X_i \iid \gmm(\mu, \alpha)$ distributed according to the GMM and let $\mathbf{Y} = (Y_1, \dots, Y_n) \sim \probdist_{\mathcal{P}_n}(A,\mu)$ be distributed according to the unstructured ASD, with $\alpha =  |A| / n$.
We empirically observe that $d_W\left( \frac{1}{n} \sum_{i=1}^n 1_{X_i}, \frac{1}{n} \sum_{i=1}^n 1_{Y_i} \right) = O(n^{-0.5})$, where $d_W$ is the $1$-Wasserstein distance, also known as the earth mover's distance (Figure \ref{fig:wasserstein}). We note that our empirical observation 
matches the result that the
Wasserstein distance between the normal distribution $N(\mu, \sigma)$ and the empirical distribution of $n$ samples from $N(\mu, \sigma)$ is also $O(n^{-0.5})$ \cite{Rippl2016,Weed2019}.

\section{Regularized MLE for Submatrix ASD}
\label{sec:append:submat_regularized}

For the submatrix family $\mathcal{M}_N$, \cite{Liu2019} show that a regularized version of the MLE is asymptotically unbiased.
Specifically, for a submatrix $M \in \mathbb{R}^{p \times q}$ of a matrix $N \in \mathbb{R}^{m \times m}$, they define the regularized scan statistic function $\Gamma_{\text{R}}(M) = \Gamma(M) - \sqrt{2\log\left(m^2 \binom{m}{p} \binom{m}{q} \right)}$ and the regularized MLE $\widehat{A}_{\text{R}} = \displaystyle\argmax_{M \in \mathcal{M}_N} \Gamma_{\text{R}}(M)$. \cite{Liu2019} then show that $\widehat{A}_{\text{R}}$ is asymptotically unbiased.

However, our proof of Theorem \ref{thm:subexp} shows that the MLE $\mle{\mathcal{M}_N}$ for the submatrix ASD, which does not use the above regularization, is also asymptotically unbiased. Thus, the regularization is not required. Empirically, we find that
that the MLE $\mle{\mathcal{M}_N}$ and the regularized MLE $\widehat{A}_{\text{R}}$ have similar bias and similar $F$-measure to the anomaly (Figure \ref{fig:submat_mle}), suggesting that the regularization proposed by \cite{Liu2019} is not necessary to reduce bias or increase performance in anomaly estimation.

\begin{figure}[t]
    \centering
        \includegraphics[width=0.48\linewidth]{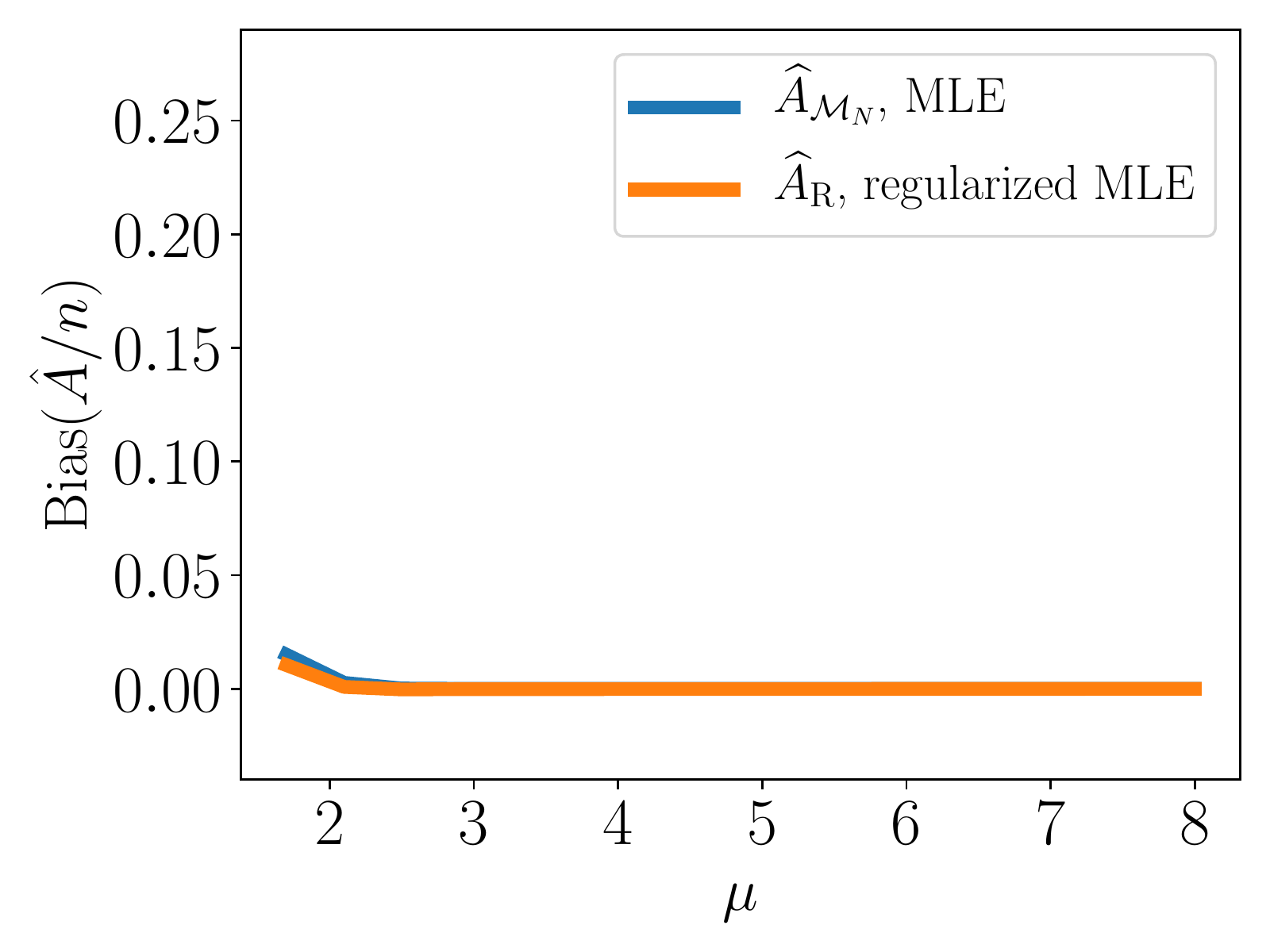}
    \includegraphics[width=0.48\linewidth]{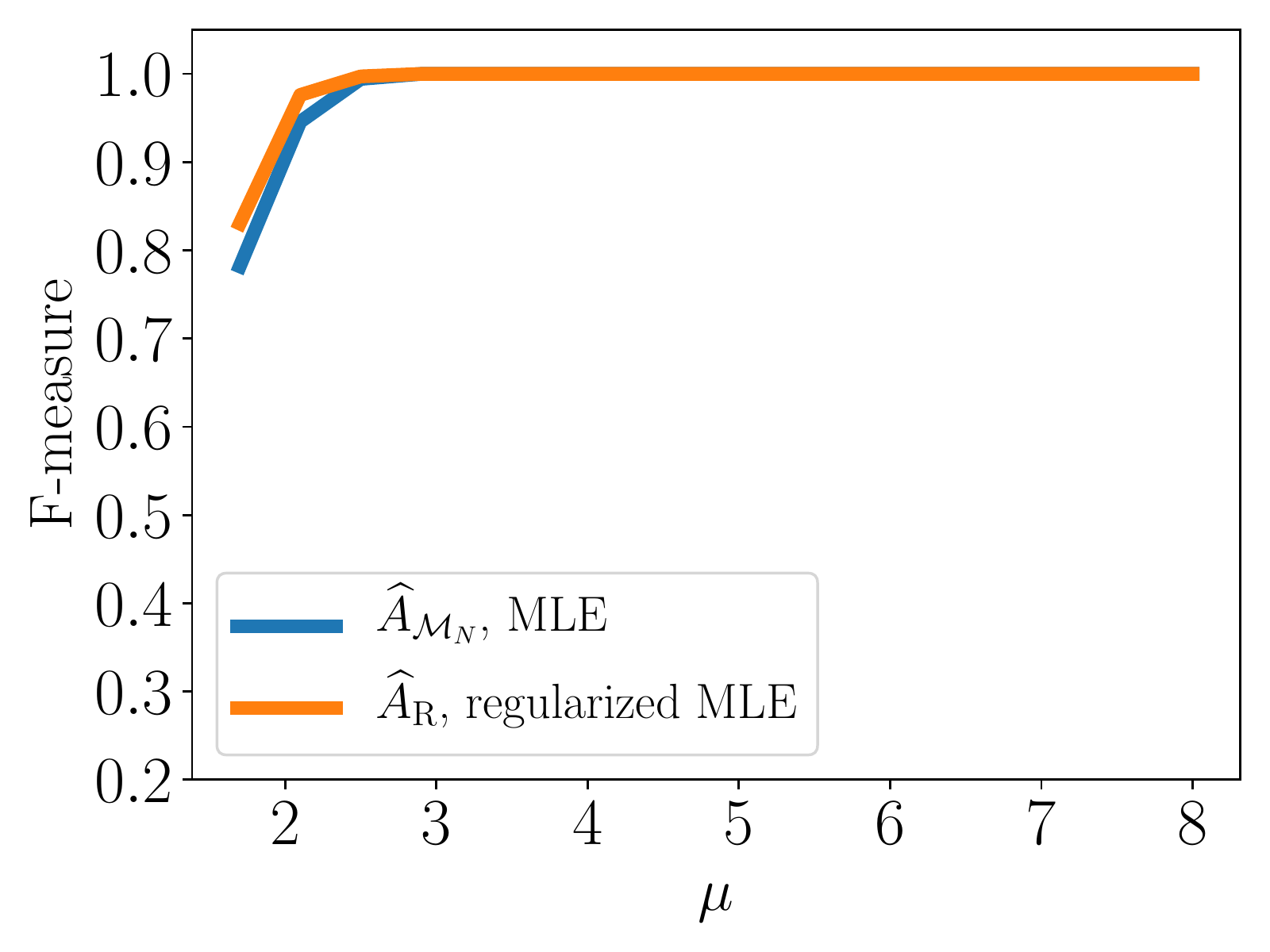}
    \caption{$\mathbf{X} \sim \probdist_{\mathcal{M}_N}(A, \mu)$ is distributed according to the submatrix ASD, where $N \in \mathbb{R}^{30\times 30}$ is a $30\times 30$ matrix. \textbf{Left: } $\text{Bias}(|\mle{\mathcal{M}_N}| / n)$ and $\text{Bias}(|\widehat{A}_{\text{R}}| / n)$ versus $\mu$ for means $\mu \geq \mu_{\text{detect}}$. \textbf{Right: } $F$-measure of $\mle{\mathcal{M}_N}$ and $\widehat{A}_{\text{R}}$ versus $\mu$ for means $\mu \geq \mu_{\text{detect}}$.}
    \label{fig:submat_mle}
\end{figure}

\section{Approximating the GMM Estimator for the Submatrix Family and the Connected Family}
\label{sec:append:gmm_approx}


For the submatrix family $\mathcal{S} = \mathcal{M}_N$ and the connected family $\mathcal{S} = \mathcal{C}_G$, our GMM estimator
\begin{equation}
\label{eq:anomaly_mix_supp}
\widehat{A}_{\text{GMM}} = \argmax_{\substack{S \in \mathcal{S} \\  \big| |S| - \widehat{\alpha}_{\text{GMM}}  \big| \leq \sqrt{\frac{\log{n}}{n}}  }}  \left( \sum_{i \in S}  \widehat{r}_i \right)
\end{equation}
can be inefficient to compute because of the constraint on the size $|S|$ of the subset $S$. In our experiments, we relax this constraint by computing the following approximation $\widetilde{A}_{\text{GMM}}$ of our GMM estimator: 
\begin{equation}
\widetilde{A}_{\text{GMM}} = \argmax_{S \in \mathcal{S}} \sum_{i \in S} (\widehat{r}_i - \tau).
\end{equation}
Here, $\tau> 0$ is a positive number that we use to ``shift" the estimated responsibilities $\widehat{r}_i$ to $\widehat{r}_i - \tau$. 
We select $\tau > 0$ so that the number $T$ of positive ``shifted" responsibilities $\widehat{r}_i - \tau$ satisfies $\big| T - \widehat{\alpha}_{\text{GMM}}  \big| \leq \sqrt{\frac{\log{n}}{n}}$. That is, $\tau$ is chosen so that $\{i : \widehat{r}_i - \tau > 0\} = T$, where $T$ satisfies $\big| T - \widehat{\alpha}_{\text{GMM}} n \big| \leq \sqrt{\frac{\log{n}}{n}}$.
 Because the number of positive shifted responsibilities is $T$, we expect our approximate estimator $\widetilde{A}_{\text{GMM}}$ to have size  $|\widetilde{A}_{\text{GMM}}|  \approx T \approx \widehat{\alpha}_{\text{GMM}} n$.

\section{Proof of Theorem \ref{thm:subexp}}
\label{sec:append:thm3_proof}

%


\subsection{Preliminary Lemmas}

We first prove the following technical lemmas.

\begin{lem}
\label{lem:whp_intersect_lem}
Let $\{A_n\}_{n=1, 2, \dots}$ and $\{B_n\}_{n=1, 2, \dots}$ be two sequences of events in the same probability space. Suppose $\displaystyle\lim_{n\to\infty} P(A_n) = 1$ and $\displaystyle\lim_{n\to\infty} P(B_n) = 1$. Then $\displaystyle\lim_{n\to\infty} P(A_n \cap B_n) = 1$.
\end{lem}

\begin{proof}
Let $p_n = P(A_n)$ and $q_n = P(B_n)$. Then
\begin{equation}
P(A_n \cap B_n) = P(A_n) + P(B_n) - P(A_n \cup B_n) = p_n + q_n - P(A_n \cup B_n) \geq p_n + q_n - 1 ,
\end{equation}
where in the last inequality we use that $P(A_n \cup B_n) \leq 1$. Thus,
\begin{equation*}
\lim_{n\to\infty} P(A_n \cap B_n) \geq \lim_{n\to\infty} (p_n+q_n - 1) = \left( \lim_{n\to\infty} p_n \right) +\left( \lim_{n\to\infty} q_n \right )  - 1= 1. 
\end{equation*}
Since $\displaystyle\lim_{n\to\infty} P(A_n \cap B_n) \leq 1$ by definition, it follows that $\displaystyle\lim_{n\to\infty} P(A_n \cap B_n) = 1$
\end{proof}

\begin{lem}
\label{lem:whp_exp_bound}
Let $X_1, X_2, \dots$ be a sequence of random variables with $X_n < 1$ for all $n$. 
If $\displaystyle\lim_{n\to\infty} P(X_n > C) = 0$ for some $C > 0$, then $E[X_n] < 2C$ for sufficiently large $n$.
\end{lem}

\begin{proof}

We have two cases depending on the value of $C$. First, suppose $C \geq 1$. Then $X_n < 1 < C$ for all $n$, and it follows that $E[X_n] < C < 2C$. 

Next, suppose $C \in (0,1)$. Let $n$ be sufficiently large so that $P(X_n> C) < \frac{C}{1-C}$. Then
\begin{equation*}
E[X_n] \leq C \cdot P(X_n \leq C) + 1 \cdot P(X_n > C) \leq C\cdot \left( 1 - \frac{C}{1-C} \right) + \frac{C}{1-C} = 2C. \qedhere 
\end{equation*}
\end{proof}

\begin{lem}
\label{lem:xa_conc_bound}
Let $X_n \sim N(\mu_n, \sigma_n)$, with $\mu_n, \sigma_n \to \infty$ as $n \to \infty$. Then
\begin{equation}
\lim_{n\to\infty} P\left( \mu_n - \sqrt{2\sigma_n \log{n}} \leq X_n \leq \mu_n + \sqrt{2\sigma_n\log{n}} \right) = 1.
\end{equation}
\end{lem}

\begin{proof}
We have
\begin{equation}
\begin{split}
P(X_n > \mu_n + \sqrt{2\sigma_n \log{n}}) 
&= P(Z > \sqrt{2\log{n}}), \text{ where $Z \sim N(0,1)$},  \\
&\leq \frac{1}{\sqrt{2\pi}} \cdot \frac{1}{\sqrt{2\log{n}}} \cdot \frac{1}{n} = O\left ( \frac{1}{n} \right ),
\end{split}
\end{equation}
where in the last inequality we use the standard bound $P(Z \geq x) \leq \frac{1}{\sqrt{2\pi}} \frac{1}{x} e^{-x^2/2}$. By symmetry, we have
\begin{equation}
P(X_n < \mu_n - \sqrt{2\sigma_n \log{n}}) \leq O\left ( \frac{1}{n} \right )
\end{equation}
Thus,
\begin{equation*}
P\left (\mu_n - \sqrt{2\sigma_n \log{n}} \leq X_n \leq \mu_n + \sqrt{2\sigma_n\log{n}}\right ) > 1 - O\left ( \frac{1}{n} \right ).
\end{equation*}
Taking the limit as $n\to\infty$ proves the result.
\end{proof}

\begin{lem}
\label{lem:gaussian_lem_size}
Suppose $X_v \iid N(0,1)$ for $v=1, \dots, n$.
Let $\mathcal{S} \subseteq \mathcal{P}_n$ be a family of subsets of $[n]$ with size $|\mathcal{S}| = \Omega(n)$. For any $k \in [n]$ define $\mathcal{S}_k = \{B \in \mathcal{S} : |B| = k\}$ and $Y_k = \displaystyle\max_{B \in \mathcal{S}_k} \left ( \sum_{v \in B} X_v \right )$. 
 Then,
\begin{equation}
    \lim_{n\to\infty} P\left(Y_{k} \leq \sqrt{2n\log{|\mathcal{S}|}} \; \text{ for all } \; k =1, \dots, \frac{n}{2} \right) = 1
\end{equation}
\end{lem}

\begin{proof}
Let $t = \sqrt{2n\log{|\mathcal{S}|}}$ and let $\Phi$ be the CDF of the standard normal distribution. Fix $k \in \left \{1, \dots, \frac{n}{2}\right \}$. We have
\begin{equation}
\begin{split}
P(Y_k > t) &= P\left (\max_{B \in \mathcal{S}_{k}} \sum_{v \in B} X_v > t\right ) \\
&\leq \sum_{B \in \mathcal{S}_{k}} P\left (\sum_{v \in B} X_v > t\right ) \\
&= |\mathcal{S}_{k}| \cdot (1-\Phi(t/\sqrt{k})) \\
&\leq |\mathcal{S}| \cdot (1-\Phi(t/\sqrt{k})).
\end{split}
\end{equation}
where the first inequality uses a union bound and the second equality uses that $\sum_{v \in B} X_v \sim N(0, k)$. Plugging in the standard bound $1-\Phi(x) \leq \frac{1}{\sqrt{2\pi}} \frac{1}{x} e^{-x^2/2}$ gives us:
\begin{equation}
\label{eq:lem_k_indiv}
\begin{split}
P\left (\max_{B \in \mathcal{S}_{k}} \sum_{v \in B} X_v > t\right ) &\leq |\mathcal{S}| \cdot (1-\Phi(t/\sqrt{k})) \\
&\leq |\mathcal{S}| \cdot \frac{1}{\sqrt{2\pi}} \frac{\sqrt{k}}{t} e^{-t^2/2k} \\
&= |\mathcal{S}| \cdot \frac{1}{\sqrt{2\pi}}  \cdot \sqrt{\frac{k}{2n\cdot \log |\mathcal{S}|}} \cdot e^{-\frac{n\cdot\log |\mathcal{S}|}{k}} \\
&= \left( \sqrt{ \frac{k}{4\pi n} } \right) \cdot \frac{1}{\sqrt{\log|\mathcal{S}|}} \cdot |\mathcal{S}|^{1 - \frac{n}{k}} \\
&\leq \left( \sqrt{ \frac{1}{4\pi} } \right) \cdot \frac{1}{|\mathcal{S}| \cdot \sqrt{\log|\mathcal{S}|}}, \text{ since $k \leq \frac{n}{2}$}.
\end{split}
\end{equation}
Taking a union bound over all $k=1, \dots, \frac{n}{2}$ gives us
\begin{equation}
\begin{split}
P\left (\max_{B \in \mathcal{S}_{k}} \sum_{v \in B} X_v > t \; \text{ for any }\; k = 1, \dots, \frac{n}{2} \right ) &\leq \sum_{k=1}^{n/2} P\left (\max_{B \in \mathcal{S}_{k}} \sum_{v \in B} X_v > t \right ) \\ 
&\leq \frac{n}{2} \cdot \left( \sqrt{ \frac{1}{4\pi} } \right) \cdot \frac{1}{|\mathcal{S}| \cdot \sqrt{\log|\mathcal{S}|}}, \text{ by Equation \eqref{eq:lem_k_indiv}}, \\
&= \left( \sqrt{ \frac{1}{16\pi} } \right) \cdot \frac{n}{|\mathcal{S}| \cdot \sqrt{\log|\mathcal{S}|}} \\
&= O\left( \frac{1}{\sqrt{\log{n}}} \right) \text{ as $|\mathcal{S}| = \Omega(n)$}.
\end{split}
\end{equation}

It follows that for sufficiently large $n$, there exists a constant $C > 0$ such that 
\begin{equation}
\label{eq:lem_k_combined}
P\left (\max_{B \in \mathcal{S}_{k}} \sum_{v \in B} X_v > t \; \text{ for any }\; k = 1, \dots, \frac{n}{2} \right ) \leq \frac{C}{\sqrt{\log{n}}},
\end{equation}
so that
\begin{equation}
\begin{split}
\lim_{n\to\infty} P\left(Y_{k} \leq \sqrt{2n\log{|\mathcal{S}|}} \; \text{ for all } \; k = 1, \dots, \frac{n}{2} \right) &= 1 - \lim_{n\to\infty} P\left (Y_{k} > \sqrt{2n\log{|\mathcal{S}|}} \; \text{ for any } \; k = 1, \dots, \frac{n}{2} \right ) \\
&\geq 1 - \lim_{n\to\infty} \left( \frac{C}{\sqrt{\log{n}}} \right), \text{ by Equation \eqref{eq:lem_k_combined}}, \\
&= 1,
\end{split}
\end{equation}
proving the result.
%
%
\end{proof}

\begin{lem}
\label{lem:mle_size_n_2}
Let $\mathbf{X} \sim \probdist_{\mathcal{S}}(A, \mu)$ where $|A|=\alpha n$ for $0 < \alpha < 0.5$. Then $\displaystyle\lim_{n\to\infty} P(|\mleS| \leq 0.5n) = 1$.
\end{lem}

\begin{proof}
Let $S \in \mathcal{S}$ be a set with size $|S| > 0.5n$. To prove the claim, it suffices to show that
\begin{equation}
\label{eq:mle_size_ub}
\frac{1}{\sqrt{|A|}}\sum_{v \in A} X_v > \frac{1}{\sqrt{|S|}}\sum_{v \in S} X_v.
\end{equation}
with high probability.

Note that $\frac{1}{\sqrt{|A|}} > \frac{1}{\sqrt{\alpha n}}$ and $\frac{1}{\sqrt{|S|}} \leq \frac{1}{\sqrt{0.5n}}$. Thus,  to prove \eqref{eq:mle_size_ub} it is sufficient to prove that
\begin{equation}
\label{eq:mle_size_ub2}
\frac{1}{\sqrt{0.25n}}\sum_{v \in A} X_v > \frac{1}{\sqrt{0.5n}}\sum_{v \in S} X_v \Longleftrightarrow \sum_{v\in A} X_v > \left( \sqrt{2\alpha} \right) \sum_{v \in S} X_v.
\end{equation}
with high probability. 

By independence of the $X_v$, we have that $\sum_{v \in A} X_v \sim N(\mu\alpha n, \alpha n)$, so by Lemma \ref{lem:xa_conc_bound} it follows that $\sum_{v \in A} X_v > \mu\alpha n - \sqrt{2\alpha n \log{n}}$ with high probability.  Similarly,  we have that $\sum_{v \in S} X_v \sim N(M, \beta n)$ where $M \leq \mu \alpha n$ (as there are at most $|A|=\alpha n$ terms in the sum $\sum_{v \in S} X_v$ with mean $\mu$, and the other terms have mean $0$). Thus, by Lemma \ref{lem:xa_conc_bound}, we also have $\sum_{v \in S} X_v < \mu\alpha n + \sqrt{2 |S| \log{n}}$ with high probability.
Putting together the lower bound on $\sum_{v \in A} X_v$ and the upper bound on $\sum_{v \in S} X_v$,  \eqref{eq:mle_size_ub2} can be reduced to
\begin{equation}
\label{eq:mle_size_ub3}
\begin{split}
\mu\alpha n - \sqrt{2\alpha n \log{n}} > \left( \sqrt{2\alpha} \right)  ( \mu\alpha n + \sqrt{2 |S| \log{n}}) &\Longleftrightarrow \left(1-\sqrt{2\alpha}\right) \mu \alpha n > \sqrt{n\log{n}} \left( \sqrt{\frac{|S|}{\alpha n}} + \sqrt{2\alpha} \right).
\end{split}
\end{equation}

Because $\alpha < 0.5$, the LHS is $\Theta(n)$. Since the RHS is $o(n)$,  then \eqref{eq:mle_size_ub3} holds with high probability. Thus, \eqref{eq:mle_size_ub} also holds with high probability, and the result follows.
\end{proof}

\subsection{Main Lemmas}

\begin{lem}
\label{lem:main_lem}
Let $\mathbf{X} \sim \probdist_{\mathcal{S}}(A, \mu)$ where $|\mathcal{S}| = \Omega(n)$ and $|A| = \alpha n$ with $0 < \alpha < 0.5$. Suppose $\displaystyle\lim_{n\to\infty} P(A \subseteq \mleS) = 1$. Then for sufficiently large $n$, we have
\begin{equation}
\label{eq:lem_bias_A_size}
\textnormal{Bias}\left(\frac{|\mleS|}{n}\right) \leq 2\alpha \left( \left( \frac{\mu\alpha n + \sqrt{2 n\log{|\containA|}} + \sqrt{2\alpha n\log{n}}}{\mu\alpha n - \sqrt{2\alpha n\log{n}} } \right)^2 - 1 \right) + o(1).
\end{equation}
\end{lem}


\begin{proof}

We will first derive the $o(1)$ term in Equation \eqref{eq:lem_bias_A_size}.
Let $X_S = \sum_{v \in S} X_v$, and
define the following events:
\begin{align*}
D_n &= \left [ |\mleS| \leq \frac{n}{2} \right] \\
E_n &= [ A \subseteq \mleS ], \\
F_n &= \left[ \mu\alpha n - \sqrt{2\alpha n \log{n}} \leq X_A \leq \mu\alpha n + \sqrt{2\alpha n \log{n}} \right], \\
G_n &= \left[ \max_{\substack{B \in \containA \\ |B| \leq \frac{n}{2}}} X_{B \setminus A} \leq \sqrt{2n \log{|\containA|}}  \right].
\end{align*}

Let $H_n = D_n \cap E_n \cap F_n \cap G_n$.
We claim that  $\displaystyle\lim_{n\to\infty} P(H_n) = 1$.

To prove this claim, first note that $\displaystyle\lim_{n\to\infty} P(D_n) = 1$ by Lemma \ref{lem:mle_size_n_2} and $\displaystyle\lim_{n\to\infty} P(E_n) = 1$ by assumption. Moreover, because $X_A \sim N(\mu\alpha n, \alpha n)$, it follows from Lemma \ref{lem:xa_conc_bound} that $\displaystyle\lim_{n\to\infty} P(F_n) = 1$. Finally, by applying Lemma \ref{lem:gaussian_lem_size} with the anomaly family $\containA$, we have that $\displaystyle\lim_{n\to\infty} P(G_n) = 1$. Thus, by a repeated application of Lemma \ref{lem:whp_intersect_lem}, we have $\displaystyle\lim_{n\to\infty} P(H_n) = \lim_{n\to\infty} P(D_n \cap E_n \cap F_n \cap G_n) = 1$.

Now define $p_n = P(H_n)$. Then, we have
\begin{equation}
\begin{split}
\label{eq:bias_hn_nohn}
\textnormal{Bias}\left(\frac{|\mleS|}{n}\right) &= p_n \cdot \textnormal{Bias}\left(\frac{|\mleS|}{n} \;\middle|\; H_n \right) + (1-p_n) \cdot \textnormal{Bias}\left(\frac{|\mleS|}{n} \;\middle|\; H_n^c \right)  \\
&\leq  \textnormal{Bias}\left(\frac{|\mleS|}{n} \;\middle|\; H_n \right) + (1-p_n) \\
&=  \textnormal{Bias}\left(\frac{|\mleS|}{n} \;\middle|\; H_n \right) + o(1),
\end{split}
\end{equation}
where in the second line we use that $p_n \leq 1$ and $\textnormal{Bias}\left(\frac{|\mleS|}{n} \;\middle|\; H_n^c \right) \leq 1$, and in the third line we use that $\displaystyle\lim_{n\to\infty} P(H_n) = 1$.

To complete the proof, we will bound $\textnormal{Bias}\left(\frac{|\mleS|}{n} \;\middle|\; H_n \right)$. Since the bias term conditions on $H_n = D_n \cap E_n \cap F_n \cap G_n$, for the rest of the proof we will assume that the events $D_n$, $E_n$, $F_n$, and $G_n$ hold.

Since $E_n$ holds, we have that
\begin{align*}
\Gamma(A) &= \frac{1}{\sqrt{|A|}} \sum_{v \in A} X_v = \frac{1}{\sqrt{\alpha n}} X_A \\[5pt]
\Gamma(\mleS) &= \frac{1}{\sqrt{|\mleS|}} \sum_{v \in \mleS} X_v = \frac{1}{\sqrt{|\mleS|}} \left (X_A + X_{\mleS \setminus A}\right ),
\end{align*}
We will find lower and upper bounds for $X_{\mleS \setminus A}$ in terms of $|\mleS|$, and use those bounds to derive  \eqref{eq:lem_bias_A_size}.

We start by finding a lower bound for $X_{\mleS \setminus A}$. Since $F_n$ holds, we have:
\begin{equation}
\label{eq:xa_conc}
\mu\alpha n - \sqrt{2\alpha n \log{n}} \leq X_A \leq \mu\alpha n + \sqrt{2\alpha n \log{n}}.
\end{equation}
Combining \eqref{eq:xa_conc} and the fact that $\Gamma(A) \leq \Gamma(\mleS)$ yields
\begin{equation}
\label{eq:thm_ub_chain}
   \frac{1}{\sqrt{\alpha n}} (\mu \alpha n - \sqrt{2\alpha n \log{n}}) \leq \Gamma(A) \leq \Gamma(\mleS) \leq \frac{1}{\sqrt{|\mleS|}} \left (\mu\alpha n + \sqrt{2\alpha n \log{n}} + X_{\mleS \setminus A}\right ).
\end{equation}
By assumption, $\mleS \setminus A \neq \varnothing$. Thus, solving for $X_{\mleS \setminus A}$ gives us a lower bound on $X_{\mleS \setminus A}$:
\begin{equation}
\label{eq:main_lem_ub}
X_{\mleS \setminus A} \geq \sqrt{\frac{|\mleS|}{\alpha n}} \left( \mu\alpha n - \sqrt{2\alpha n\log{n}}\right) - \mu\alpha n - \sqrt{2\alpha n\log{n}}.
\end{equation}

Next, we find an upper bound for $X_{\mleS \setminus A}$.  Since $D_n$ holds, we have $|\mleS| \leq \frac{n}{2}$. Since $G_n$ also holds, we have
\begin{equation}
\label{eq:main_lem_lb}
    X_{\mleS \setminus A} \leq \max_{\substack{B \in \containA \\ |B| \leq \frac{n}{2}}} X_{B \setminus A} \leq \sqrt{2n \log{|\containA|}}.
\end{equation}
Combining the lower bound from
\eqref{eq:main_lem_ub} and the upper bound from \eqref{eq:main_lem_lb} yields
\begin{equation}
\label{eq:main_lem_lb_ub}
\sqrt{\frac{|\mleS|}{\alpha n}} \left( \mu\alpha n - \sqrt{2\alpha n\log{n}}\right) - \mu\alpha n - \sqrt{2\alpha n\log{n}} \leq X_{\mleS \setminus A} \leq \sqrt{2n \log{|\containA|}}
\end{equation}
Thus, the LHS of \eqref{eq:main_lem_lb_ub} is less than the RHS of \eqref{eq:main_lem_lb_ub}, i.e. 
\begin{equation}
\label{eq:main_lem_lb_ub2}
\sqrt{\frac{|\mleS|}{\alpha n}} \left( \mu\alpha n - \sqrt{2\alpha n\log{n}}\right) - \mu\alpha n - \sqrt{2\alpha n\log{n}} \leq \sqrt{2n \log{|\containA|}}.
\end{equation}
Rearranging \eqref{eq:main_lem_lb_ub2} yields
\begin{equation}
\frac{|\mleS|}{n} - \alpha \leq \alpha \left( \left( \frac{\mu\alpha n + \sqrt{2 n\log{|\containA|}} + \sqrt{2\alpha n\log{n}}}{\mu\alpha n - \sqrt{2\alpha n\log{n}} }  \right)^2 - 1 \right).
\end{equation}
So by Lemma \ref{lem:whp_exp_bound}, we have
\begin{equation}
\label{eq:bias_mles_hn}
\textnormal{Bias}\left(\frac{|\mleS|}{n} \;\middle|\; H_n \right) = E\left[ \frac{|\mleS|}{n} - \alpha \;\middle|\; H_n \right] \leq 2\alpha \left( \left( \frac{\mu\alpha n + \sqrt{2 n\log{|\containA|}} + \sqrt{2\alpha n\log{n}}}{\mu\alpha n - \sqrt{2\alpha n\log{n}} }  \right)^2 - 1 \right).
\end{equation}
The result follows by combining Equations \eqref{eq:bias_hn_nohn} and \eqref{eq:bias_mles_hn}.
\end{proof}

\begin{lem}
\label{lem:paper_lem}
Let $\mathbf{X} \sim \probdist_{\mathcal{S}}(A, \mu)$ where $|\mathcal{S}| = \Omega(n)$ and $|A| = \alpha n$ with $0 < \alpha < 0.5$. Assume $\displaystyle\lim_{n\to\infty} P(A \subseteq \mleS) = 1$. If $\text{Bias}(|\mleS| / n) \geq \gamma$, then
\begin{equation}
|\containA| \geq (C_{\mu, \gamma, \alpha})^n \cdot e^{-\Theta(\sqrt{n \log n})}
\end{equation}
for sufficiently large $n$, where
$C_{\mu, \alpha, \gamma} = \exp\left( \frac{1}{2} \mu^2 \alpha^2\left( \sqrt{1+\frac{\gamma}{4\alpha}} - 1 \right)^2   \right).$
\end{lem}

\begin{proof}
By Lemma \ref{lem:main_lem}, we have
\begin{equation}
\label{eq:gam_bias_bound}
\gamma \leq \text{Bias}(|\mleS| / n) \leq 2\alpha \left( \left( \frac{\mu\alpha n + \sqrt{2 n\log{|\containA|}} + \sqrt{2\alpha n\log{n}}}{\mu\alpha n - \sqrt{2\alpha n\log{n}} }  \right)^2 - 1 \right) + o(1).
\end{equation}
Thus, the LHS of \eqref{eq:gam_bias_bound} is less than the RHS of \eqref{eq:gam_bias_bound}, i.e.
\begin{equation}
\label{eq:gam_bias_bound2}
\gamma \leq 2\alpha \left( \left( \frac{\mu\alpha n + \sqrt{2 n\log{|\containA|}} + \sqrt{2\alpha n\log{n}}}{\mu\alpha n - \sqrt{2\alpha n\log{n}} }  \right)^2 - 1 \right) + o(1)
\end{equation}
Let $n$ be sufficiently large so that the $o(1)$ term in \eqref{eq:gam_bias_bound2} is less than $\frac{\gamma}{2}$. Then, solving for $|\containA|$ in \eqref{eq:gam_bias_bound2}:
\begin{align*}
&\; \gamma \leq 2\alpha \left( \left( \frac{\mu\alpha n + \sqrt{2 n\log{|\containA|}} + \sqrt{2\alpha n\log{n}}}{\mu\alpha n - \sqrt{2\alpha n\log{n}} }  \right)^2 - 1 \right) + \frac{\gamma}{2} \\[7.5pt]
&\Rightarrow \sqrt{\frac{\gamma}{4\alpha} + 1} \leq \frac{\mu\alpha n + \sqrt{2 n\log{|\containA|}} + \sqrt{2\alpha n\log{n}}}{\mu\alpha n - \sqrt{2\alpha n\log{n}} } \\[7.5pt]
&\Rightarrow \mu\alpha n \left( \sqrt{\frac{\gamma}{4\alpha} + 1} - 1 \right) - \Theta(\sqrt{n \log{n}}) \leq \sqrt{2n \log{|\containA|}} \\[7.5pt]
&\Rightarrow  |\containA| \geq \left[\exp\left( \frac{1}{2} \mu^2 \alpha^2 \left( \sqrt{\frac{\gamma}{4\alpha} + 1} - 1 \right)^2\right) \right]^n \cdot e^{-\Theta(\sqrt{n\log{n}})}
\end{align*}
completing the proof.
\end{proof}

\subsection{Proof of Theorem}

Using the above lemmas, we are now ready to prove Theorem \ref{thm:subexp}.

\begin{namedproblem}[\textbf{Theorem \ref{thm:subexp}}]
Let $\mathbf{X} = (X_1, \dots, X_n) \sim \probdist_{\mathcal{S}}(A, \mu)$ where $\mathcal{S} = \Omega(n)$ and $|A| = \alpha n$ with $0 < \alpha < 0.5$.
Suppose $|\containA|$ is sub-exponential in $n$ and $\displaystyle\lim_{n\to\infty} P(A \subseteq \mleS) = 1$.
Then $\displaystyle\lim_{n\to\infty} \textnormal{Bias}(|\mleS| / n) = 0$.
\end{namedproblem}

\begin{proof}
Let $\gamma > 0$ and let $C_{\mu, \alpha, \gamma}$ be as defined in Lemma \ref{lem:paper_lem}. Note that because $C_{\mu, \alpha, \gamma} > 1$, we have $\frac{2C_{\mu, \alpha, \gamma}}{1+C_{\mu, \alpha, \gamma}} > 1$. 

Because $|\containA|$ is sub-exponential, there exists sufficiently large $n$ so that
\begin{equation}
\label{eq:thm_containA_bound}
|\containA| < \left( \frac{2C_{\mu, \alpha, \gamma}}{1+C_{\mu, \alpha, \gamma}}  \right)^n.
\end{equation}
We also note that because $C_{\mu,\alpha,\gamma} > 1$, then $\frac{2}{1+C_{\mu, \alpha, \gamma}} < 1$. For sufficiently large $n$,  $e^{-\Theta(\sqrt{   \log{n} / n }) }$ will get arbitrarily close to $1$.  Thus, we have 
\begin{equation}
\label{eq:thm_containA_bound2}
e^{-\Theta(\sqrt{   \log{n} / n }) } \geq \frac{2}{1+C_{\mu, \alpha, \gamma}}.
\end{equation}
Combining both \eqref{eq:thm_containA_bound} and \eqref{eq:thm_containA_bound2} gives us
\begin{equation}
C_{\mu, \alpha, \gamma}^n \cdot e^{-\Theta(\sqrt{n \log n})} = \left( C_{\mu, \alpha, \gamma} \cdot e^{-\Theta(\sqrt{   \log{n} / n }) } \right)^n \geq  \left( C_{\mu, \alpha, \gamma} \cdot \frac{2}{1+C_{\mu, \alpha, \gamma}} \right)^n  > |\containA|,
\end{equation}
for sufficiently large $n$, where the first inequality follows by \eqref{eq:thm_containA_bound2} and the second inequality follows by \eqref{eq:thm_containA_bound}.

Thus, by the contrapositive of Lemma \ref{lem:paper_lem}, it follows that $\text{Bias}(|\mleS|/n) < \gamma$ for sufficiently large $n$. Taking the limit as $n\to\infty$ yields
\begin{equation}
\label{eq:lem_gamma_bias}
\lim_{n\to\infty} \text{Bias}(|\mleS|/n) \leq \gamma.
\end{equation}

Because \eqref{eq:lem_gamma_bias} holds for all $\gamma > 0$, it follows that $\displaystyle\lim_{n\to\infty} \text{Bias}(|\mleS|/n)  \leq 0$. Furthermore, because  $\displaystyle\lim_{n\to\infty} P(A \subseteq \mleS) = 1$, we also have $\displaystyle\lim_{n\to\infty} \text{Bias}(|\mleS|/n) \geq 0$. Thus, $\displaystyle\lim_{n\to\infty} \text{Bias}(|\mleS|/n)  = 0$, as desired.
\end{proof}

\section{Proof of Theorem \ref{thm:unconstr}}

In the following proof, we slightly abuse notation and assume that all statements of the form $\displaystyle\lim_{n\to\infty} R_n = Y$, where $R_n$ and $Y$ are random variables, hold almost surely.

\begin{namedproblem}[\textbf{Theorem \ref{thm:unconstr}}]
Let $\mathbf{X} \sim \probdist_{\mathcal{P}_n}(A, \mu)$ where $|A|=\alpha n$ with $0 < \alpha < 0.5$.
Then $\displaystyle\lim_{n\to\infty} \textnormal{Bias}(|\mle{\mathcal{P}_n}|/n) > 0$.
\end{namedproblem}

\begin{proof}
 From Proposition \ref{thm:glr_mle_asd}, we have
\begin{equation}
\mle{\mathcal{P}_n} = \argmax_{S \subseteq [n]} \frac{1}{\sqrt{|S|}} \sum_{v \in S} X_v.
\end{equation}

Because the maximum is taken over all subsets of $[n]$, an equivalent formulation of the above is $\mle{\mathcal{P}_n} = \{v : X_v > \mleT\}$, where
\begin{equation}
\label{eq:unconstr_thm_T}
\mleT = \argmax_{T \in \mathbb{R}} \left ( \frac{1}{\sqrt{|\{v \in [n] : X_v > T\}}|} \sum_{v \in [n] : X_v > T} X_v \right ).
\end{equation}

We start by showing that $\displaystyle\lim_{n\to\infty} \mleT$ is finite. To do so, we will find an expression for the RHS as $n\to\infty$. Let $M_T = \{v \in [n] : X_v > T, v \in A\}$ and $N_T = \{v \in [n] : X_v > T, v \not\in A\}$. Additionally, let $\nu_{\mu, T}$ be the mean of a $N(\mu, 1)$ distribution that is truncated to be above $T$. Then
\begin{equation}
\label{eq:unconstr_thm_max}
\sum_{v \in [n]: X_v > T} X_v = \left ( \frac{\sum_{v \in M_T} X_v}{|M_T|} \right ) \cdot |M_T| + \left ( \frac{\sum_{v \in N_T} X_v}{|N_T|} \right ) \cdot |N_T|
\end{equation}

By the strong law of large numbers, $\displaystyle\lim_{n\to\infty} \frac{\sum_{v \in M_T} X_v}{|M_T|} = \nu_{\mu, T}$ and $\displaystyle\lim_{n\to\infty} \frac{\sum_{v \in M_T} X_v}{|M_T|} = \nu_{0, T}$. Similarly, $\displaystyle\lim_{n\to\infty} \frac{|M_T|}{\alpha n\cdot(1-\Phi(T-\mu))} = 1$ and $\displaystyle\lim_{n\to\infty} \frac{|N_T|}{(1-\alpha)n\cdot(1-\Phi(T))} = 1$,  where $\Phi$ is the CDF of a standard normal. Plugging these limits into \eqref{eq:unconstr_thm_max} gives us

\begin{equation}
\label{eq:unconstr_thm_num}
\lim_{n\to\infty} \frac{\sum_{v \in [n]: X_v > T} X_v}{\nu_{\mu, T} \cdot (\alpha n (1-\Phi(T-\mu)) + \nu_{0, T} \cdot ((1-\alpha)n(1-\Phi(T)) } = 1.
\end{equation}
A similar calculation yields
\begin{equation}
\label{eq:unconstr_thm_denom}
\lim_{n\to\infty} \frac{ \sqrt{|\{v \in [n] : X_v > T\}|} }{ \sqrt{ \alpha n (1-\Phi(T-\mu)) + (1-\alpha) n (1-\Phi(T)) } } = 1.
\end{equation}
Plugging in Equations \eqref{eq:unconstr_thm_num} and \eqref{eq:unconstr_thm_denom} into Equation \eqref{eq:unconstr_thm_T} yields
\begin{equation}
\begin{split}
\lim_{n\to\infty} \mleT &= \lim_{n\to\infty} \left[ \argmax_{T \in \mathbb{R}} \left ( \frac{\nu_{\mu, T} \cdot (\alpha n (1-\Phi(T-\mu)) + \nu_{0, T} \cdot ((1-\alpha)n(1-\Phi(T))}{\sqrt{ \alpha n (1-\Phi(T-\mu)) + (1-\alpha) n (1-\Phi(T)) }} \right ) \right] \\[5pt]
&= \lim_{n\to\infty} \left[ \argmax_{T \in \mathbb{R}} \left ( \frac{\nu_{\mu, T} \cdot (\alpha (1-\Phi(T-\mu)) + \nu_{0, T} \cdot ((1-\alpha)(1-\Phi(T))}{\sqrt{ \alpha (1-\Phi(T-\mu)) + (1-\alpha) (1-\Phi(T)) }} \cdot \sqrt{n} \right ) \right] \\[5pt]
&= \argmax_{T \in \mathbb{R}} \left ( \frac{\nu_{\mu, T} \cdot (\alpha (1-\Phi(T-\mu)) + \nu_{0, T} \cdot ((1-\alpha)(1-\Phi(T))}{\sqrt{ \alpha (1-\Phi(T-\mu)) + (1-\alpha) (1-\Phi(T)) }} \right ).
\end{split}
\end{equation}
Thus $\displaystyle\lim_{n\to\infty} \mleT$ is finite. 

Next, define $T^* = \displaystyle\lim_{n\to\infty} \mleT$.
To complete the proof, we use $T^*$ to derive an expression for $\displaystyle\lim_{n\to\infty} \frac{|\mle{\mathcal{P}_n}|}{n}$, and then use that expression to bound $\displaystyle\lim_{n\to\infty} \text{Bias}\left (\frac{|\mle{\mathcal{P}_n}|}{n} \right )$.

Since the fraction of observations $X_i$ such that $i \in A$ and
$X_i > \mleT$ is asymptotically $\frac{|A|}{n} \cdot (1-\Phi(\mleT-\mu))$, it follows that
\begin{equation}
\label{eq:eq_mleS_1}
\lim_{n\to\infty} \frac{|A \cap \mle{\mathcal{P}_n}|}{n} = \lim_{n\to\infty} \left( \frac{|A|}{n} \cdot (1-\Phi(\mleT-\mu)) \right) = \alpha \cdot (1-\Phi(T^*-\mu)).
\end{equation}
Similarly, the fraction of observations $X_i$ such that $i \not \in A$ and $X_i > \mleT$ is asymptotically $\left(1-\frac{|A|}{n}\right) \cdot (1-\Phi(\mleT))$, so we have
\begin{equation}
\label{eq:eq_mleS_2}
\lim_{n\to\infty} \frac{|A^c\cap \mle{\mathcal{P}_n}|}{n} = \lim_{n\to\infty} \left( \left(1-\frac{|A|}{n}\right) \cdot (1-\Phi(\mleT)) \right) = (1-\alpha) \cdot (1-\Phi(T^*)).
\end{equation}
Combining Equations \eqref{eq:eq_mleS_1} and \eqref{eq:eq_mleS_2} gives us
\begin{equation}
\lim_{n\to\infty} \frac{|\mle{\mathcal{P}_n}|}{n}  = \alpha \cdot (1-\Phi(T^*-\mu)) + (1-\alpha) \cdot (1-\Phi(T^*))
\end{equation}
Thus, the asymptotic bias of the MLE $\mleS$ is
\begin{equation*}
\begin{split}
\lim_{n\to\infty} \text{Bias}(|\mle{\mathcal{P}_n}| / n) &= \lim_{n\to\infty} E[|\mle{\mathcal{P}_n}| / n ] - \alpha \\[5pt]
&= \alpha \cdot (1-\Phi(T^*-\mu)) + (1-\alpha) \cdot (1-\Phi(T^*)) - \alpha.
\end{split}
\end{equation*}
Since the above expression is always positive for $\alpha < 0.5$, so it follows that $\displaystyle\lim_{n\to\infty} \text{Bias}(|\mle{\mathcal{P}_n}| / n) > 0$.
\end{proof}

\section{Proof of Theorem \ref{thm:gmm_likelihood} and Corollaries \ref{cor:gmm1}, \ref{cor:gmm2}}

To prove Theorem \ref{thm:gmm_likelihood}, we require the following Lemma.

\begin{lem}
\label{lem:high_prob_norm}
Let $X_1, \dots, X_n \iid N(\mu,\sigma^2)$. Then 
\begin{equation}
P\left(\max_{i\in [n]} X_i \leq \mu + 2\sigma\sqrt{\log{n}}\right) \geq 1 - \frac{1}{n}.
\end{equation}
\end{lem}

\begin{proof}
For fixed $i \in [n]$ we have
\begin{equation}
\begin{split}
P(X_i > \mu + 2\sigma\sqrt{\log{n}}) &= P(Z > 2\sqrt{\log{n}}), \text{where $Z \sim N(0,1)$} \\
&\leq \frac{1}{\sqrt{2\pi}} \cdot \frac{1}{2\sqrt{\log{n}}} \cdot \frac{1}{n^2} < \frac{1}{n^2}
\end{split}
\end{equation}
where in the second line we use the standard bound $P(Z \geq x) \leq \frac{1}{\sqrt{2\pi}}\cdot \frac{1}{x} e^{-x^2/2}$.  Thus, $P(X_i \leq \mu + 2\sigma\sqrt{\log{n}}) \geq 1-\frac{1}{n^2}$.
By a union bound, it follows that
\begin{equation}
P(X_i \leq \mu + 2\sigma\sqrt{\log{n}} \text{\; \; for all $i \in [n]$}) \geq 1- \frac{1}{n},
\end{equation}
which implies the desired result.
\end{proof}


\begin{namedproblem}[\textbf{Theorem \ref{thm:gmm_likelihood}}]
Let $\mathbf{X} = (X_1, \dots, X_n) \sim \asdS$, where $|A|=\alpha n$ for $0 < \alpha < 0.5$ and $\mu \geq C\sqrt{\log{n}}$ for a sufficiently large constant $C > 0$. For sufficiently large $n$,  we have that
\begin{equation*}
|\alphaGMM - \alpha| \leq \sqrt{\frac{\log{n}}{n}} \quad \text{ and } \quad |\muGMM - \mu| \leq 3\sqrt{\frac{\log{n}}{n}}
\end{equation*}
with probability at least $ 1-\frac{1}{n}$.
\end{namedproblem}

\allowdisplaybreaks
\begin{proof}
For $\widetilde{\alpha} \in (0,1)$ and $\widetilde{\mu} > 0$, let $L_{\widetilde{\alpha}, \widetilde{\mu}}(x) = \log\left(   \widetilde{\alpha} \cdot \exp\left(-\frac{1}{2} (x-\widetilde{\mu})^2 \right) + (1-\widetilde{\alpha}) \cdot \exp\left(-\frac{1}{2} x^2 \right)   \right)$ be the (scaled) log-likelihood function for the mixture distribution $\widetilde{\alpha}\cdot N(\widetilde{\mu}, 1) + (1-\widetilde{\alpha})\cdot N(0,1)$, and define $L_{\widetilde{\alpha}, \widetilde{\mu}}(\mathbf{X}) = \prod_{i=1}^n L_{\widetilde{\alpha}, \widetilde{\mu}}(X_i)$. Then 
\begin{equation}
\widehat{\alpha}_{\text{GMM}}, \widehat{\mu}_{\text{GMM}} = \argmax_{\substack{\widetilde{\alpha} \in (0,1) \\ \widetilde{\mu} \in (0, \infty)}} L_{\widetilde{\alpha}, \widetilde{\mu}}(\mathbf{X})
\end{equation}

To prove the claim, it suffices to show that if $|\widehat{\alpha}_{\text{GMM}} - \widetilde{\alpha}| > \sqrt{\frac{\log{n}}{n}}$ or $|\widehat{\mu}_{\text{GMM}} - \widetilde{\mu}| > 3 \sqrt{\frac{\log{n}}{n}}$, then $L_{\alpha, \mu}(\mathbf{X}) > L_{\widetilde{\alpha}, \widetilde{\mu}}(\mathbf{X})$ with probability at least $1-\frac{1}{n}$.


We will prove the following equivalent formulation: if $\kappa, \tau$ are real numbers such that $|\kappa| > \sqrt{\frac{\log{n}}{n}}$ or $|\tau| > 3 \sqrt{\frac{\log{n}}{n}}$, then
\begin{equation}
L_{\alpha, \mu}(\mathbf{X}) > \likX
\end{equation}
with probability at least $1-\frac{1}{n}$.

We proceed by a case analysis based on whether $\kappa$ and $\tau$ also satisfy the following additional conditions:
\begin{align}
\label{eq:cond_kappa_tau_1}
\frac{1}{n^2} < \alpha + \kappa &< 1 - \frac{1}{n^2}, \\[5pt]
\label{eq:cond_kappa_tau_2}
\frac{\mu^2}{\tau^2} &> 100.
\end{align}

Briefly, the intuition for the above conditions is that if $\kappa$ and $\tau$ satisfy \eqref{eq:cond_kappa_tau_1} and \eqref{eq:cond_kappa_tau_2}, then we can derive a simplified formula for the likelihood  $\likX$. 

In all cases, we assume that $\mu - 2\sqrt{\log{n}} \leq X_i \leq \mu + 2\sqrt{\log{n}}$ for $i \in A$ and $- 2\sqrt{\log{n}} \leq X_i \leq 2\sqrt{\log{n}}$ for $i \not\in A$, as these events hold with probability at least $1-\frac{1}{n}$ by Lemma \ref{lem:high_prob_norm}. 

\noindent \textbf{Case 1: $\kappa$ satisfies \eqref{eq:cond_kappa_tau_1} and $\tau$ satisfies \eqref{eq:cond_kappa_tau_2}.} The log-likelihood $\likX$ can be written as
\begin{equation}
\label{eq:loglik_X}
\begin{split}
\likX &= \log\left( \prod_{i=1}^n  \left( (\alpha+\kappa) \cdot \exp\left( -\frac{1}{2} (X_i - \mu - \tau)^2 \right) + (1-\alpha-\kappa) \cdot \exp\left( -\frac{X_i^2}{2} \right) \right) \right) \\
&= \sum_{i=1}^n \log\left( (\alpha+\kappa) \cdot \exp\left( -\frac{1}{2} (X_i - \mu - \tau)^2 \right) + (1-\alpha-\kappa) \cdot \exp\left( -\frac{X_i^2}{2} \right) \right)
\end{split}
\end{equation}

Let $T_1 = (\alpha+\kappa) \cdot \exp\left( -\frac{1}{2} (X_i - \mu - \tau)^2\right)$ be the first term in the logarithm and let $T_2 = (1-\alpha-\kappa) \cdot \exp\left( -\frac{X_i^2}{2}\right)$ be the second term. 
We claim that if $i \in A$, then $T_1 +  T_2 = (1+ o(n^{-1}))T_1$, while if $i \not\in A$ then $T_1 +  T_2 = (1+ o(n^{-1}))T_2$.

To show that $T_1 +  T_2 = (1+  o(n^{-1}))T_1$ for $i \in A$, we compute $\frac{T_2}{T_1}$:
\begin{equation}
\allowdisplaybreaks
\begin{split}
\frac{T_2}{T_1} &= \frac{ (1-\alpha-\kappa) \cdot \exp\left( -\frac{X_i^2}{2}\right) }{ (\alpha+\kappa) \cdot \exp\left( -\frac{1}{2} (X_i - \mu - \tau)^2\right) } \\[5pt]
&= \left( \frac{1-\alpha-\kappa}{\alpha+\kappa}  \right) \cdot \exp\left( -X_i (\mu + \tau) + \frac{1}{2}(\mu+\tau)^2 \right) \\[5pt]
&\leq \left( \frac{1-\alpha-\kappa}{\alpha+\kappa}  \right) \cdot \exp\left( -(\mu - 2\sqrt{\log{n}}) (\mu + \tau) + \frac{1}{2}(\mu+\tau)^2 \right)\\[5pt]
&= \left( \frac{1-\alpha-\kappa}{\alpha+\kappa}  \right) \cdot \exp\left( -\frac{1}{2}\mu^2 + \frac{1}{2}\tau^2 + 2\mu\sqrt{\log{n}} + 2\tau\sqrt{\log{n}} \right) \\[5pt]
&\leq n^2 \cdot \exp\left( -\frac{1}{2}\mu^2 + \frac{1}{2}\left(\frac{\mu^2}{100}\right) + 2\mu\sqrt{\log{n}} + 2\left(\frac{\mu}{10}\right)\sqrt{\log{n}} \right) \\
&= n^2 \cdot \exp\left( -\frac{99}{200}\mu^2 + \frac{11}{5}\mu\sqrt{\log{n}}  \right),
\end{split}
\end{equation}
where the first inequality uses that $X_i \geq \mu - 2\sqrt{\log{n}}$ and the second inequality uses that $\tau \leq \frac{\mu}{10}$ (which follows from \eqref{eq:cond_kappa_tau_2}).

Now $-\frac{99}{200}t^2 + \frac{11}{5}t\sqrt{\log{n}}$ is a concave quadratic with a maximum at $t=\frac{20}{9}\sqrt{\log{n}}$. Since $\mu \geq C \sqrt{\log{n}} > \frac{20}{9}\sqrt{\log{n}}$ (for sufficiently large $C$), it follows that $-\frac{99}{200}\mu^2 + \frac{11}{5}\mu\sqrt{\log{n}} \leq -\frac{99}{200}(C\sqrt{\log{n}})^2 + \frac{11}{5}(C\sqrt{\log{n}})\cdot\sqrt{\log{n}}$. Plugging this into the above equation yields
\begin{equation}
\allowdisplaybreaks
\begin{split}
n^2 \cdot \exp\left( -\frac{99}{200}\mu^2 + \frac{11}{5}\mu\sqrt{\log{n}}  \right) &\leq n^2\cdot \exp\left(  -\frac{99}{200}(C\sqrt{\log{n}})^2 + \frac{11}{5}(C\sqrt{\log{n}})\cdot\sqrt{\log{n}}    \right) \\
&= n^2\cdot n^{-\frac{99}{200}C^2 + \frac{440}{200}C     } \\
&= o(n^{-1}) \text{ for sufficiently large $C$}.
\end{split}
\end{equation}

Thus, $\frac{T_2}{T_1} =  o(n^{-1})$, which implies $T_1+T_2 = (1+o(n^{-1}))T_1$ for $i \in A$.
By a similar derivation, we also have that $T_1 + T_2 = (1+o(n^{-1})) T_2$ for $i \not\in A$. 

Using these relationships between $T_1$ and $T_2$, we rewrite the log-likelihood in Equation \eqref{eq:loglik_X} as
\begin{equation}
\label{eq:likX_kt}
\begin{split}
\likX &= \sum_{i\in A} \log\left( (1+o(n^{-1})) \cdot T_1    \right) +  \sum_{i\not\in A} \log\left( (1+o(n^{-1})) \cdot T_2    \right) \\
&= \sum_{i\in A} \log\left( (1+o(n^{-1})) \cdot (\alpha + \kappa) \exp\left( -\frac{1}{2} (X_i - \mu - \tau)^2\right)    \right) +  \sum_{i\not\in A} \log\left( (1+o(n^{-1})) \cdot (1-\alpha-\kappa)   \exp\left( -\frac{X_i^2}{2}\right)    \right) \\
&= \alpha n \cdot \log(\alpha+\kappa) + (1-\alpha)n \cdot \log(1-\alpha-\kappa) - \frac{1}{2} \sum_{i \in A} (X_i - \mu-\tau)^2 - \frac{1}{2} \sum_{i\not\in A} X_i^2 + o(1).
\end{split}
\end{equation}

Plugging in $\kappa=\tau=0$ into \eqref{eq:likX_kt} yields the following expression for the log-likelihood $L_{\alpha, \mu}(\mathbf{X})$ with the true parameters $\alpha, \mu$:
\begin{equation}
\label{eq:likX_00}
\begin{split}
L_{\alpha, \mu}(\mathbf{X}) = \alpha n \cdot \log{\alpha} + (1-\alpha)n \cdot \log(1-\alpha) - \frac{1}{2} \sum_{i \in A} (X_i - \mu)^2 - \frac{1}{2} \sum_{i \not\in A} X_i^2 + o(1).
\end{split}
\end{equation}

So after equating equations \eqref{eq:likX_kt} and \eqref{eq:likX_00} and simplifying, we have that $L_{\alpha, \mu}(\mathbf{X}) > \likX$ is equivalent to
\begin{align*}
L_{\alpha, \mu}&(\mathbf{X}) > \likX \\
&\Leftrightarrow \alpha\log\left( \frac{\alpha}{\alpha+\kappa} \right) + (1-\alpha) \log\left( \frac{1-\alpha}{1-\alpha-\kappa}   \right) + \frac{\tau \alpha}{2}\left(\tau - 2 \frac{\sum_{i \in A} (X_i - \mu)}{\alpha n} \right) + o(n^{-1}) > o(n^{-1}),
\end{align*}
To prove that $L_{\alpha, \mu}(\mathbf{X}) > \likX$ and complete the proof,  it suffices to show that
\begin{equation}
\label{eq:Lkt_great_L00_ineq}
\alpha\log\left( \frac{\alpha}{\alpha+\kappa} \right) + (1-\alpha) \log\left( \frac{1-\alpha}{1-\alpha-\kappa}   \right) + \frac{\tau \alpha}{2}\left(\tau - 2 \frac{\sum_{i \in A} (X_i - \mu)}{\alpha n} \right) = \Omega(n^{-1}).
\end{equation}

To bound the above inequality, we first note that the first two terms $\alpha\log\left( \frac{\alpha}{\alpha+\kappa} \right) + (1-\alpha) \log\left( \frac{1-\alpha}{1-\alpha-\kappa}   \right)$ are the KL-divergence between a $\text{Bern}(\alpha)$ random variable and a $\text{Bern}(\alpha+\kappa)$ random variable.  By Pinsker's inequality, we have
\begin{equation}
\begin{split}
\alpha\log\left( \frac{\alpha}{\alpha+\kappa} \right) + (1-\alpha) \log\left( \frac{1-\alpha}{1-\alpha-\kappa}   \right) &= D_{KL}\left(\text{Bern}(\alpha) \mid \mid \text{Bern}(\alpha+\kappa)\right) \\
&\geq 2 \left[ d_{TV}\big(\text{Bern}(\alpha), \text{Bern}(\alpha+\kappa)\big) \right]^2 \\
&= 2\kappa^2.
\end{split}
\end{equation}
Second, we note that $\frac{\sum_{i \in A} (X_i - \mu)}{\alpha n} \sim N(0, \frac{1}{\alpha n})$, so by Lemma \ref{lem:high_prob_norm}, we have that $\frac{\sum_{i \in A} (X_i - \mu)}{\alpha n} < \sqrt{\frac{\log{n}}{\alpha n}}$ with high probability.
Thus,
\begin{align}
\label{eq:lb_kappa_tau}
\alpha\log\left( \frac{\alpha}{\alpha+\kappa} \right) + (1-\alpha) \log\left( \frac{1-\alpha}{1-\alpha-\kappa}   \right) + \frac{\tau \alpha}{2}\left(\tau - 2 \frac{\sum_{i \in A} (X_i - \mu)}{\alpha n} \right) \geq 2 \kappa^2 + \frac{\tau\alpha}{2} \left( \tau - 2\sqrt{\frac{\log{n}}{n}} \right)
\end{align}
To prove that the RHS of \eqref{eq:lb_kappa_tau} is $\Omega(n^{-1})$, we use casework depending on whether $|\kappa| > \sqrt{\frac{\log{n}}{n}}$ or $|\tau| > 3 \sqrt{\frac{\log{n}}{n}}$. 

\textbf{Case 1, Sub-case 1: $|\kappa| > \sqrt{\frac{\log{n}}{n}}$}.

We bound the RHS of \eqref{eq:lb_kappa_tau} as
\begin{equation}
\begin{split}
2 \kappa^2 + \frac{\tau\alpha}{2} \left( \tau - 2\sqrt{\frac{\log{n}}{n}} \right) &\geq 2 \frac{\log{n}}{n}+ \frac{\tau\alpha}{2} \left( \tau - 2\sqrt{\frac{\log{n}}{n}} \right) \\
&\geq  2 \frac{\log{n}}{n}- \frac{\alpha}{2} \cdot \frac{\log{n}}{n} \\
&= \frac{\log{n}}{n} \cdot \left( 2 - \frac{\alpha}{2} \right) \\
&= \Omega(n^{-1}),
\end{split}
\end{equation}
where the second inequality uses that $\frac{\tau\alpha}{2} \left( \tau - 2\sqrt{\frac{\log{n}}{n}} \right)$ is a quadratic in $\tau$ whose minimum is $-\frac{\alpha}{2} \frac{\log{n}}{n}$.

\textbf{Case 1, Sub-case 2: $|\tau| > 3\sqrt{\frac{\log{n}}{n}}$}.

Note that the condition on $\tau$ implies that either $\tau > 3 \sqrt{\frac{\log{n}}{n}}$ or $\tau < -3 \sqrt{\frac{\log{n}}{n}}$.

We also note that $\frac{\tau\alpha}{2} \left( \tau - 2\sqrt{\frac{\log{n}}{n}} \right)$ is a quadratic in $\tau$ that is decreasing for $\tau < \sqrt{\frac{\log{n}}{n}} $ and is increasing for $\tau > \sqrt{\frac{\log{n}}{n}} $.  Depending on the value of $\tau$, we lower bound $\frac{\tau\alpha}{2} \left( \tau - 2\sqrt{\frac{\log{n}}{n}} \right)$ as follows:
\begin{align*}
\tau > 3 \sqrt{\frac{\log{n}}{n}} &\Longrightarrow \frac{\tau\alpha}{2} \left( \tau - 2\sqrt{\frac{\log{n}}{n}} \right) > \frac{\alpha}{2} \left( 3 \sqrt{\frac{\log{n}}{n}} \right) \left( 3 \sqrt{\frac{\log{n}}{n}} - 2 \sqrt{\frac{\log{n}}{n}} \right) = \frac{3\alpha}{2} \frac{\log{n}}{n} \\[5pt]
\tau < - 3 \sqrt{\frac{\log{n}}{n}}  &\Longrightarrow \frac{\tau\alpha}{2} \left( \tau - 2\sqrt{\frac{\log{n}}{n}} \right) > \frac{\alpha}{2} \left( -3 \sqrt{\frac{\log{n}}{n}} \right) \left( -3 \sqrt{\frac{\log{n}}{n}} - 2 \sqrt{\frac{\log{n}}{n}} \right) = \frac{15\alpha}{2} \frac{\log{n}}{n}.
\end{align*}

In either case, we have that $\frac{\tau\alpha}{2} \left( \tau - 2\sqrt{\frac{\log{n}}{n}} \right) > \frac{3\alpha}{2} \frac{\log{n}}{n}$. Thus the RHS of \eqref{eq:lb_kappa_tau} is
\begin{equation}
2 \kappa^2 + \frac{\tau\alpha}{2} \left( \tau - 2\sqrt{\frac{\log{n}}{n}} \right) \geq \frac{3\alpha}{2} \frac{\log{n}}{n} = \Omega(n^{-1}),
\end{equation}
as desired.

\noindent \textbf{Case 2: $\kappa$ does not satisfy \eqref{eq:cond_kappa_tau_1}, $\tau$ satisfies \eqref{eq:cond_kappa_tau_2}.}  This means that  either $\alpha + \kappa < \frac{1}{n^2} $ or $\alpha + \kappa > 1-\frac{1}{n^2}$. We will treat each of these sub-cases separately.  

Before doing so, we require the following lower bound on $L_{\alpha, \mu}(\mathbf{X})$: for sufficiently large $n$, $L_{\alpha, \mu}(\mathbf{X}) > -n$. To prove this lower bound, from \eqref{eq:likX_00} we have
\begin{equation}
\label{eq:likX_00_bound}
\begin{split}
L_{\alpha, \mu}(\mathbf{X}) &= n \cdot \log(1+o(n^{-1})) + \alpha n \cdot \log{\alpha} + (1-\alpha)n \cdot \log(1-\alpha) - \frac{1}{2} \sum_{i \in A} (X_i - \mu)^2 - \frac{1}{2} \sum_{i \not\in A} X_i^2 \\
&= n\left( \log(1+o(n^{-1})) + H(\alpha) - \frac{\alpha}{2}\left( \frac{\sum_{i \in A} (X_i - \mu)^2}{\alpha n}  \right) - \frac{1-\alpha}{2} \left( \frac{\sum_{i \not\in A} X_i^2}{(1-\alpha) n} \right) \right) \\
&\geq n\left( - \frac{\alpha}{2}\left( \frac{\sum_{i \in A} (X_i - \mu)^2}{\alpha n}  \right) - \frac{1-\alpha}{2} \left( \frac{\sum_{i \not\in A} X_i^2}{(1-\alpha) n} \right)  \right).
\end{split}
\end{equation}
where $H(\alpha)$ is the binary entropy function.  By standard tail bounds on $\chi^2$ random variables (see e.g. Lemma 1 of \cite{Laurent2000}), we have that $\frac{\sum_{i \in A} (X_i - \mu)^2}{\alpha n} \leq 2$ and  $\frac{\sum_{i \not\in A} (X_i - \mu)^2}{(1-\alpha) n} \leq 2$ with probability at least $1-\frac{1}{n^3}$, for sufficiently large $n$.  Plugging these upper bounds into \eqref{eq:likX_00_bound} yields
\begin{equation}
L_{\alpha, \mu}(\mathbf{X}) \geq n\left( - \frac{\alpha}{2}\left( \frac{\sum_{i \in A} (X_i - \mu)^2}{\alpha n}  \right) - \frac{1-\alpha}{2} \left( \frac{\sum_{i \not\in A} X_i^2}{(1-\alpha) n} \right)  \right) \geq n\left( -\alpha - (1-\alpha) \right) = -n.
\end{equation}

\noindent \textbf{Case 2, Sub-case 1: $\alpha + \kappa < \frac{1}{n^2} $}.

Our strategy for this sub-case, as well as the subsequent ones, will be to upper bound $\likX$ and show that $\likX < -n \leq L_{\alpha, \mu}(\mathbf{X})$. 

We have
\begin{equation}
\begin{split}
\likX &= \sum_{i=1}^n \log\left( (\alpha+\kappa) \cdot \exp\left( -\frac{1}{2} (X_i - \mu - \tau)^2 \right) + (1-\alpha-\kappa) \cdot \exp\left( -\frac{X_i^2}{2} \right) \right) \\
&\leq \sum_{i\in A} \log\left( (\alpha+\kappa) \cdot \exp\left( -\frac{1}{2} (X_i - \mu - \tau)^2 \right) + (1-\alpha-\kappa) \cdot \exp\left( -\frac{X_i^2}{2} \right) \right).
\end{split}
\end{equation}

We upper bound the first term $(\alpha+\kappa) \cdot \exp\left( -\frac{1}{2} (X_i - \mu - \tau)^2 \right)$ in the logarithm by
\begin{equation}
\label{eq:ub1_case1_subcase1}
(\alpha+\kappa) \cdot \exp\left( -\frac{1}{2} (X_i - \mu - \tau)^2 \right) \leq \alpha+\kappa < \frac{1}{n^2}.
\end{equation}

We upper bound the second term $(1-\alpha-\kappa) \cdot \exp\left( -\frac{X_i^2}{2} \right)$ in the logarithm as
\begin{equation}
\label{eq:ub2_case1_subcase1}
\begin{split}
(1-\alpha-\kappa) \cdot \exp\left( -\frac{X_i^2}{2} \right) &\leq \exp\left( -\frac{X_i^2}{2} \right) \\
&\leq \exp\left( -\frac{1}{2} (\mu - 2\sqrt{\log{n}})^2   \right) \\
&\leq \exp\left( -\frac{1}{2} (C\sqrt{\log{n}} - 2\sqrt{\log{n}})^2  \right), \\
&= n^{-(C-2)^2/2} \\
&\leq n^{-2} \text{ for sufficiently large $C$},
\end{split}
\end{equation}

where the first inequality follows from the assumption that,  $X_i > \mu - 2\sqrt{\log{n}}$ and the second inequality follows the fact that $\mu \geq C\sqrt{\log{n}} > 2\sqrt{\log{n}}$ and $-\frac{1}{2} (\mu - 2\sqrt{\log{n}})^2$ is decreasing for $\mu > 2\sqrt{\log{n}}$.

Combining the upper bounds in \eqref{eq:ub1_case1_subcase1} and \eqref{eq:ub2_case1_subcase1} gives us the following upper bound on $\likX$:
\begin{equation}
\likX \leq \sum_{i \in A} \log(n^{-2} + n^{-2}) = \alpha n (\log(2n^{-2})) = -\Theta(n\log{n}).
\end{equation}
Thus, for sufficiently large $n$, we have that $\likX \leq \alpha n (\log(2n^{-2})) < -n = L_{\alpha, \mu}(\mathbf{X})$, as desired.

\noindent \textbf{Case 2, Sub-case 2: $\alpha + \kappa > 1 - \frac{1}{n^2} $}. 

As in the previous sub-case, we will prove that $\likX < -n \leq L_{\alpha, \mu}(\mathbf{X})$.

Since \eqref{eq:cond_kappa_tau_2} holds, we have that $\frac{\mu^2}{\tau^2} > 100$, or equivalently $\frac{9}{10}\mu < \mu+\tau < \frac{11}{10} \mu$. We upper bound $\likX$ as
\begin{equation}
\begin{split}
\likX &\leq \sum_{i\not\in A} \log\left( (\alpha+\kappa) \cdot \exp\left( -\frac{1}{2} (X_i - \mu - \tau)^2 \right) + (1-\alpha-\kappa) \cdot \exp\left( -\frac{X_i^2}{2} \right) \right) \\
&\leq \sum_{i \not\in A}  \log\left( \exp\left( -\frac{1}{2} (X_i - (\mu + \tau))^2 \right) +  \frac{1}{n^2}    \right) \\
&\leq \sum_{i \not\in A}  \log\left( \exp\left( -\frac{1}{2} ( 2\sqrt{\log{n}} - \frac{9}{10}C\sqrt{\log{n}})^2 \right) +  \frac{1}{n^2}    \right)  \\
&= \sum_{i \not\in A}  \log\left( n^{-\frac{1}{2} \left(2-\frac{9}{10}C\right)^2 } +  n^{-2}    \right) \\
&\leq  \sum_{i \not\in A}  \log\left( n^{-2 } +  n^{-2}    \right) \text{ for sufficiently large $C$} \\
&= (1-\alpha) n \log(2n^{-2}),
\end{split}
\end{equation}

where the second inequality follows from the assumption that $1-\frac{1}{n^2} < \alpha + \kappa < 1$, and the third inequality follows from  the fact that $X_i \leq 2\sqrt{\log{n}} \leq \frac{9}{10} C\sqrt{\log{n}} \leq \frac{9}{10}\mu \leq \mu + \tau$ for sufficiently large $C$.

For sufficiently large $n$, we have that $\likX \leq (1-\alpha) n \log(2n^{-2}) = -\Theta(n\log{n}) < -n \leq L_{\alpha, \mu}(\mathbf{X})$, as desired.

\noindent \textbf{Case 3: $\tau$ does not satisfy \eqref{eq:cond_kappa_tau_2}.}

Since $\tau$ does not satisfy \eqref{eq:cond_kappa_tau_2}, we have that either $\mu + \tau > \frac{11}{10}\mu$ or $\mu + \tau < \frac{9}{10}\mu$. We treat each of these sub-cases separately.  In each sub-case, we use the bound $L_{\alpha, \mu}(\mathbf{X}) > -n$ derived in Case 2.

\noindent \textbf{Case 3, Sub-case 1: $\mu + \tau > \frac{11}{10}\mu$.} 

As before, we will show that $\likX < -n \leq L_{\alpha, \mu}(\mathbf{X})$. We upper bound $\likX$ as
\begin{equation}
\begin{split}
\likX &= \sum_{i=1}^n \log\left( (\alpha+\kappa) \cdot \exp\left( -\frac{1}{2} (X_i - \mu - \tau)^2 \right) + (1-\alpha-\kappa) \cdot \exp\left( -\frac{X_i^2}{2} \right) \right) \\
&\leq \sum_{i \in A}  \log\left( \exp\left( -\frac{1}{2} \big(X_i - (\mu + \tau)\big)^2 \right) +  \exp\left( -\frac{X_i^2}{2} \right)    \right) \\
&\leq \sum_{i \in A}  \log\left( \exp\left( -\frac{1}{2} \left(\mu + 2\sqrt{\log{n}}- \frac{11}{10}\mu\right)^2 \right) +  \exp\left( -\frac{1}{2}(\mu - 2\sqrt{\log{n}})^2 \right)    \right)\\
&= \sum_{i \in A}  \log\left( \exp\left( -\frac{1}{2} \left(\frac{1}{10}\mu - 2\sqrt{\log{n}}\right)^2 \right) +  \exp\left( -\frac{1}{2}(\mu - 2\sqrt{\log{n}})^2 \right)    \right) \\
&\leq \sum_{i \in A}  \log\left( \exp\left( -\frac{1}{2} \left(\frac{1}{10}C\sqrt{\log{n}} + 2\sqrt{\log{n}}\right)^2 \right) +  \exp\left( -\frac{1}{2}(C\sqrt{\log{n}} - 2\sqrt{\log{n}})^2 \right)    \right) \\
&= \sum_{i \in A} \log\left( n^{-\frac{1}{2} \left(\frac{C}{10}+2\right)^2} + n^{-\frac{1}{2} (C-2)^2} \right) \\
&\leq \sum_{i \in A} \log\left( n^{-2} + n^{-2}\right) \text{ for sufficiently large $C$},
\end{split}
\end{equation}

The first inequality uses that $\alpha + \kappa \leq 1$ and $1-\alpha - \kappa \leq 1$. The second inequality uses that $X_i \leq \mu + 2\sqrt{\log{n}} < \frac{11}{10}\mu < \mu + \tau$ (where $\mu + 2\sqrt{\log{n}} < \frac{11}{10}\mu \Leftrightarrow \mu \geq 20 \sqrt{\log{n}}$ holds for sufficiently large $C$), and $X_i \geq \mu - 2\sqrt{\log{n}}$. The third inequality uses that $\frac{1}{10}\mu \geq \frac{1}{10} C \sqrt{\log{n}} > 2\sqrt{\log{n}}$ and $\mu \geq C \sqrt{\log{n}} > 2\sqrt{\log{n}}$ (for sufficiently large $C$). 

Thus,  for sufficiently large $n$ we have $\likX \leq -\alpha n \log(2n^{-2}) < -n = L_{\alpha, \mu}(\mathbf{X})$, as desired.

\noindent \textbf{Case 3, Sub-case 2: $\mu + \tau < \frac{9}{10}\mu$.} 

As before, we will show that $\likX < -n \leq L_{\alpha, \mu}(\mathbf{X})$. We upper bound $\likX$ as

\begin{equation}
\begin{split}
\likX &= \sum_{i=1}^n \log\left( (\alpha+\kappa) \cdot \exp\left( -\frac{1}{2} (X_i - \mu - \tau)^2 \right) + (1-\alpha-\kappa) \cdot \exp\left( -\frac{X_i^2}{2} \right) \right) \\
&\leq \sum_{i\in A} \log\left( \exp\left( -\frac{1}{2} \big(X_i - (\mu + \tau)\big)^2 \right) + \exp\left( -\frac{X_i^2}{2} \right) \right) \\
&\leq  \sum_{i\in A} \log\left( \exp\left( -\frac{1}{2} \left( (\mu - 2\sqrt{\log{n}}) - \frac{9}{10}\mu  \right)^2 \right) + \exp\left( -\frac{1}{2}(\mu - 2\sqrt{\log{n}})^2 \right)      \right) \\
&= \sum_{i\in A} \log\left( \exp\left( -\frac{1}{2} \left( \frac{1}{10}\mu - 2\sqrt{\log{n}}  \right)^2 \right) + \exp\left( -\frac{1}{2}(\mu - 2\sqrt{\log{n}})^2 \right)      \right) \\
&\leq  \sum_{i\in A} \log\left( \exp\left( -\frac{1}{2} \left(\frac{1}{10}C\sqrt{\log{n}} - 2\sqrt{\log{n}}\right)^2 \right) + \exp\left( -\frac{1}{2}(C\sqrt{\log{n}} - 2\sqrt{\log{n}})^2 \right)      \right) \\
&= \sum_{i \in A}  \log\left(  n^{-\frac{1}{2} \left( \frac{C}{10} - 2 \right)^2} + n^{-\frac{1}{2} \left( C - 2 \right)^2}   \right) \\
&\leq \sum_{i \in A}  \log\left( n^{-2} + n^{-2}\right) \text{ for sufficiently large $C$}.
\end{split}
\end{equation}

The first inequality uses that $\alpha + \kappa \leq 1$ and $1-\alpha-\kappa \leq 1$.  The second inequality uses that $\mu + \tau < \frac{9}{10}\mu < \mu - 2\sqrt{\log{n}} \leq X_i$ (where the middle inequality $\frac{9}{10}\mu < \mu - 2\sqrt{\log{n}} \Leftrightarrow \mu > 20\sqrt{\log{n}}$ holds for sufficiently large $C$) and $X_i > \mu - 2\sqrt{\log{n}}$. The third inequality uses that $\frac{1}{10}\mu > \frac{1}{10} C \sqrt{\log{n}} > 2\sqrt{\log{n}}$ and $\mu > C\sqrt{\log{n}} > 2 \sqrt{\log{n}}$ for sufficiently large $C$.

So as before, for sufficiently large $n$ we have $\likX \leq -\alpha n \log(2n^{-2}) < -n = L_{\alpha, \mu}(\mathbf{X})$, as desired.
\end{proof}

\subsection{Proofs of Corollaries}

\begin{namedproblem}[\textbf{Corollary \ref{cor:gmm1}}]
Let $\mathbf{X} = (X_1, \dots, X_n) \sim \asdS$, where $|A|=\alpha n$ for $0 < \alpha < 0.5$ and $\mu \geq C\sqrt{\log{n}}$ for a sufficiently large constant $C > 0$. Then $\displaystyle\lim_{n\to\infty} \text{Bias}(|\mleS|/n) = 0$.
\end{namedproblem}

\begin{proof}
Let $B_n$ be the event that $\Big| \frac{|\mleS|}{n} - |A|\Big| \leq \sqrt{\frac{\log{n}}{n}}$. By Theorem \ref{thm:gmm_likelihood}, $P(B_n) \geq 1-\frac{1}{n}$. Note that when $B_n$ does not hold, then  $\Big| \frac{|\mleS|}{n} - |A|\Big| \leq 1$. So we have
\begin{equation}
\begin{split}
| \text{Bias}(|\mleS|/n) |  &\leq E\left[ \Big| \frac{|\mleS|}{n} - |A|\Big| \mid B_n  \right] + E\left[ \Big| \frac{|\mleS|}{n} - |A|\Big| \mid B_n^c  \right] \\
&\leq \left( \sqrt{\frac{\log{n}}{n}} \right) \cdot P(B_n) + 1\cdot P(B_n^c) \\
&\leq \sqrt{\frac{\log{n}}{n}}  + \frac{1}{n}.
\end{split}
\end{equation}

It follows that $\displaystyle\lim_{n\to\infty} \text{Bias}(|\mleS|/n) = \displaystyle\lim_{n\to\infty} \left( \sqrt{\frac{\log{n}}{n}}  + \frac{1}{n} \right) = 0$.
\end{proof}

\begin{namedproblem}[\textbf{Corollary \ref{cor:gmm2}}]
Let $\mathbf{X} = (X_1, \dots, X_n) \sim \text{ASD}_{\mathcal{P}_n}(A, \mu)$, where $|A|=\alpha n$ for $0 < \alpha < 0.5$ and $\mu \geq C \sqrt{\log{n}}$ for a sufficiently large constant $C> 0$. Then $\frac{|A \triangle \widehat{A}_{\text{GMM}}|}{|A|} \leq \frac{2}{\alpha} \sqrt{\frac{\log{n}}{n}} = o(1)$ with probability at least $1-\frac{1}{n}$.
\end{namedproblem}

\begin{proof}

By Lemma \ref{lem:high_prob_norm},  with probability at least $1-\frac{1}{n}$ we have that $X_i \leq 2\sqrt{\log{n}}$ for $i \not\in A$ and $X_j \geq (C+2)\sqrt{\log{n}}$ for $j \in A$. Thus $X_i \leq 2\sqrt{\log{n}} < (C+2)\sqrt{\log{n}} \leq X_j$ for all $i \not\in A$ and $j \in A$, which means $A$ consists of the $\alpha n$ largest observations $X_i$.

Moreover, because the responsibilities $\widehat{r}_i$ are sorted in the same order as the observations $X_i$, we have that $\widehat{A}_{\text{GMM}}$ consists of the $|\widehat{A}_{\text{GMM}}|$ largest observations $X_i$. Thus,  by Theorem \ref{thm:gmm_likelihood}, we have
\begin{equation}
|\widehat{A}_{\text{GMM}} \triangle A| = \Big| |\widehat{A}_{\text{GMM}}| - \alpha n\Big| \leq \Big| |\widehat{A}_{\text{GMM}}| - \widehat{\alpha}_{\text{GMM}} n\Big| + |\widehat{\alpha}_{\text{GMM}} - \alpha| \leq 2\sqrt{n \log n}.
\end{equation}
with probability at least $1-\frac{1}{n}$, for sufficiently large $n$.
Since $|A| = \alpha n$, it follows that $\frac{|\widehat{A}_{\text{GMM}} \triangle A|}{|A|} \leq \frac{2}{\alpha} \sqrt{\frac{\log{n}}{n}} = o(1)$ as desired.
\end{proof}

\end{document}